\pdfoutput=1

\documentclass{article}

\usepackage{graphicx}
\usepackage{amssymb,amsthm,amsmath,nicefrac}
\usepackage{tikz,placeins}
\usepackage{url}
\usepackage[capitalize,noabbrev]{cleveref}
\usepackage{booktabs,dsfont}
\usepackage{natbib}
\setcitestyle{authoryear,open={[},close={]}} 

\begin{document}
\title{An alternative proof of the vulnerability of retrieval in high intrinsic dimensionality neighborhood}
\author{Teddy Furon\\
Univ Rennes, Inria, CNRS, IRISA - France}

\maketitle

\theoremstyle{plain}
\newtheorem{theorem}{Theorem}[section]
\newtheorem{proposition}[theorem]{Proposition}
\newtheorem{lemma}[theorem]{Lemma}
\newtheorem{corollary}[theorem]{Corollary}
\newtheorem{assumption}[theorem]{Assumption}
\theoremstyle{definition}
\newtheorem{definition}[theorem]{Definition}
\theoremstyle{remark}
\newtheorem{remark}[theorem]{Remark}



\def \real{\mathbb{R}}
\def \Prob{\mathbb{P}}
\def \Exp{\mathbb{E}}
\def \Var{\mathbb{V}}
\def \x{\mathbf{x}}
\def \z{\mathbf{z}}
\def \q{\mathbf{q}}
\def \t{\mathbf{t}}
\def \y{\mathbf{y}}
\def \kt{k_t}
\def \kx{k_x}
\def \ky{k_y}
\def \dt{t}
\def \dx{x}
\def \dy{y}
\def \ie{\textit{i.e.}}
\def \D{\mathcal{D}}

\def \eg {\textit{e.g.}}
\def \rr {0.85}

\begin{abstract}
This paper investigates the vulnerability of the nearest neighbors search, which is a pivotal tool in data analysis and machine learning. 
The vulnerability is gauged as the relative amount of perturbation that an attacker needs to add onto a dataset point in order to modify its neighbor rank 
w.r.t. a query.
The statistical distribution of this quantity is derived from simple assumptions.
Experiments on six large scale datasets validate this model up to some outliers which are explained in term of violations of the assumptions.
\end{abstract}


\section{Introduction}
The search of Nearest Neighbors (NN) not only appears in the top-10 algorithms in data mining according to~\citep{wu2008top}, but it is also at the root of many machine learning methods ranging from detection~\citep{gu2019statistical}, classification~\citep{pmlr-v108-yang20b}, and regression~\citep{10.2307/2242230}.
\citet{8384208} explain its popularity because it is a flexible, non parametric, interpretable, and computationally efficient method.   

As such, many security and safety-oriented applications rely on information retrieval operated by the NN search.
The existence of an attacker intentionally circumventing the NN search becomes a plausible threat. 
For instance, social networks often compare uploaded content to a collection of materials to detect `forbidden' content (terrorism, pedopornography, fake news, \textit{etc}) or to identify copyright holders (content monetization). The detection of outliers to spot fraud in financial transactions or intrusion in security systems is another example of critical applications of the NN search~\citep{gu2019statistical}.

The NN search recently went to the rescue of other machine learning techniques.
At training time, \citet{pmlr-v119-bahri20a} use the NN search to sanitize the tranining set detecting mislabeled data, while \citet{ijcai2021-437} prevent poisoning attacks.
At test time, \citet{nguyen2018adversarial} use it as a defense against adversarial examples to robustify regression.  
\citet{Cohen_2020_CVPR}, \citet{papernot2018deep}, and~\citet{ma2018characterizing} share the same idea to detect adversarial images by considering them as outliers.
They gauge the confidence of the classifier according to the neighbors of the input in the training set in some activation layer of the neural network. 
Note that this defense has been under scrutiny in the works of~\citet{sitawarin2019robustness} and~\citet{pmlr-v80-athalye18a}.
\citet{wang2018analyzing} analyze the theoretical robustness of classification based on NN search against adversarial attacks and propose in~\citep{pmlr-v108-yang20b} to sieve the training set to prune out isolated points that cause vulnerability.
 
The versatility and ubiquity of the NN search motivate research on its intrinsic vulnerability, independently of any machine learning procedure.
It is of utmost importance to analyze the security of this key primitive before its deployment in applications where security matters.  
Our article follows the approach funded by~\citet{9194069}.
It especially deals with vectors (features extracted from objects) living in high dimensional space and compared with the Euclidean distance.
It quantifies the amount of perturbation in this feature space needed to delude the retrieval mechanism in terms of modification of the nearest neighbors.
This is exemplified by the scenario where the attacker modifies the $\kx$ nearest neighbor in the dataset to reduce its neighboring rank to a target $\kt<\kx$.
 
Under some broad assumptions, our contribution models the necessary amount of perturbation as a random variable and reveals its precise distribution depending on ranks $\kx$ and $\kt$, the size $n$ of the dataset, and on a scalar index $\ell$ characterising the distances of the dataset points in the neighborhood. 
Section~\ref{sec:Stat} shows that this r.v. converges to a simple distribution when the dataset size $n$ grows to infinity.
Experiments in Sect.~\ref{sec:Experiment} over six large scale datasets (four of them gather one billion dataset points each) of very different nature shows a very close fit with the asymptotical model.
Section~\ref{sec:Discussion} extends results to a plurality scenarios where the query (or a dataset point) is modified so that the $\kx$-NN becomes the $\kt$-NN: if $\kt< \kx$, the perturbated point gets closer to the anchor, if $\kt>\kx$ it is moved farther away. 
 
\section{Related Work}

\subsection{Experimental Evidence of the Vulnerability}
The forgery of query image deflecting nearest neighbor search has a long history in Content-Based Image Retrieval (CBIR) and image classification.

\citet{do2012enlarging} motivate their work by pinpointing that CBIR became a technology brick used in applications where security is at stake, like content monetization on social networks.
In the pre-era of deep learning, these authors delude such a mechanism (\ie\ a given dataset image  is no longer returned by the $k$-NN search)
by exploiting the Achilles' heel of the SIFT local description (orientation).
\citet{tolias2019targeted} play the opposite game with nowadays deep learning representation: the query image (\eg\ depicting a flower) is manipulated to retrieve images matching another visual content (\eg\ the Eiffel tower) to protect the privacy of the querying user. \citet{Xiao_2021_CVPR} investigate the same scenario but with black-box transferability.
The general poor level of security is due to two weak points: \textit{i}) deep neural networks provide discriminative but sensitive feature vectors easily corrupted by image manipulation, a well-known fact  in the literature of adversarial examples~\citep{DBLP:journals/corr/SabourCFF15}, and \textit{ii}) a small move in this high dimensional feature space may be sufficient to dramatically change the nearest neighbors as demonstrated by~\citet{9194069}.

Our paper is dedicated to this later vulnerability.

\subsection{Theoretical Evidence of the Vulnerability}
\citet{wang2018analyzing} extend the theoretical analysis of classification based on nearest neighbors.
The classic theory shows how the probability of correct classification behaves for fixed $k$ and $n$ increasing \citep[Chap.~5]{Devroye:1996wq} or the universal consistency when $k$ is allowed to grow with $n$ s.t. $\nicefrac{k}{n}\to0$ \citep[Chap.~6]{Devroye:1996wq}. 
\citet{wang2018analyzing} generalize this to the adversarial risk by introducing the concept of astuteness: not only a query $\q$ should be correctly classified, but so are all the points surrounding it at a distance lower than a robustness radius $r$. They outline the vulnerability of NN classifier in the following way: zero robustness is expected for a fixed $k$ and one has to set $k$ to a very large amount $\Omega(\sqrt{dn\log n})$ to asymptotically reach the robustness of the optimal Bayesian classifier.  

\citet{9194069} are the only authors we are aware of studying the core robustness of the NN search \emph{out of any supervision task}.
There is no label and thus no law $p(Y|X)$, but just a collection of $n$ points in $\mathbb{R}^d$.
They do not make assumption about the distribution of the points around the query but only about their distances.
They show that, asymptotically as the size of the dataset increases,
the necessary relative amount of perturbation has a simple expression.
Their Theorem~2 states that if the perturbation amount is bigger than this lower bound, then the attacker is sure to reach his goal.

Our work offers a much more accurate model of the amount of perturbation revealing its asymptotical and non-asymptotical distributions.
Indeed, our results contradict the Theorem~2 of~\citep{9194069}.


\section{Problem Statement}
Let us consider a dataset $\D$ of $n$ points $\{\x_i\}_{i=1}^n\in\real^d$, plus a query point $\q$ wherever in this space.
Suppose that an artificial intelligent program bases its decision on the dataset points neighboring the query.
Especially, the decision will take dataset point $\x$ into account if it is one of the $\kt$ nearest neighbors of $\q$.

Suppose that this is not the case for that query $\q$ and this dataset point $\x$.
In other words, $\x$ is only the $\kx$ nearest neighbor of $\q$ with $1\leq \kt < \kx \leq n$.
This work is interested in the amount of perturbation to be applied to dataset point $\x$ so that the program now takes it into account.

\subsection{Definition of the Relative Amount of Perturbation}
As shown in Fig.~\ref{fig:PbStat}, we denote by $\t$ the dataset point which is the $\kt$ nearest neighbor of $\q$.
We also introduce $\dt := \|\t - \q\|$ and $\dx := \|\x - \q\|$ (Euclidean norm) ; we have $\dt < \dx$.
The perturbation pushes $\x$ to a new point $\y$ with the goal that its neighbor rank is $\ky\leq \kt$, thus at a distance $\dy:=\|\y-\q\| \leq \dt$ away from $\q$.

The relative amount of perturbation is evaluated by the ratio $\delta := \|\x - \y\|/ \|\x - \q\|$, with $\delta\in(0,1)$.
It measures the norm of the perturbation $\y-\x$ relatively to the original distance $x$ of $\x$ from $\q$.
\citet{DBLP:journals/corr/SabourCFF15} also use this metric.
Obviously, this ratio is set to the minimum if $\x$ is pushed onto $\y$ in a direct line towards $\q$ so that $\y$ belongs to the line segment $\overline{\x\q}$ as shown in Fig.~\ref{fig:PbStat}.
This implies that:
\begin{equation}
\label{eq:Equivalence}
\ky \leq \kt \quad\Longleftrightarrow\quad \dy \leq \dt \quad\Longleftrightarrow \quad\delta \geq 1 - \frac{\dt}{\dx}.
\end{equation}
The sequel assumes that this strategy is feasible and we set $\delta = 1-\nicefrac{\dt}{\dx}$.

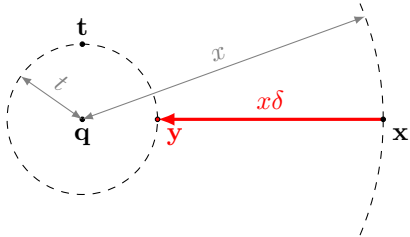
\begin{figure}[ht]
\begin{center}
\def\ra{4}
\begin{tikzpicture}
\draw [dashed] (0,0) circle (1) ;
\draw [>=latex,<-,color=red,very thick] (1,0) -- (\ra,0) node[midway,above,sloped] {$x\delta$} ;
\draw [fill=red] (1,0) circle (0.03) ;
\draw (1,0) node[below right,color=red] {$\y$};
\draw [fill] (0,0) circle (0.03) ;
\draw (0,0) node[below]{$\q$};
\draw [fill] (0,1) circle (0.03) ;
\draw (0,1) node[above]{$\t$};
\draw [fill] (\ra,0) circle (0.03) ;
\draw (\ra,0) node[below right]{$\x$};
\draw [>=latex,<->,color=gray] (0,0) -- (145:1) node[midway,above,sloped] {$\dt$} ;
\draw [>=latex,<->,color=gray] (0,0) -- (20:\ra) node[midway,above,sloped] {$\dx$} ;
\draw [dashed] (\ra,0) arc (0:23:\ra);
\draw [dashed] (\ra,0) arc (0:-23:\ra);

\end{tikzpicture}
\end{center}
\caption{Configuration of the dataset neighbors $\x$ and $\t$, and the manipulated point $\y$ with respect to the query $\q$.}
\label{fig:PbStat}
\end{figure}

\subsection{Statistical Model}
The relative amount of perturbation depends on the configuration of the dataset points locally around $\q$.
\citet{9194069} consider that this dataset of points is indeed random.
This means that distances $t$ and $x$ are occurrences of random variables $T_n$ and $X_n$ (note that $0<T_n<X_n$),
and so is the relative amount of perturbation $\delta$ w.r.t. random variable $\Delta_n:= 1 - T_n/X_n$. 
The subscript $n$ stresses that the size of the dataset is a major factor.
 
\citet{9194069} do not impose a specific distribution of the dataset points in $\real^d$, but only assume a statistical distribution of their distances from $\q$.
\begin{assumption}
\label{ass:1}
The distances of the dataset points w.r.t. the query $\q$ are random variables independent and identically distributed,
whose c.d.f. is denoted by $F:\real^+ \to [0,1]$ and p.d.f. $f:\real^+ \to \real^+$.
\end{assumption}
We should have written $F_\q$ and $f_\q$ to outline that this statistical model only holds for the distances from $\q$ ; we omit this subscript for the sake of simplicity. 

A second assumption deals with the behaviour of the c.d.f. $F$ around 0.
\begin{assumption}
\label{ass:2}
The distances of the dataset points w.r.t. the query $\q$ are \textbf{absolutely continuous} r. v., whose c.d.f. $F$ is a regularly varying function~\cite{Bingham:1987vz} around $0^+$ with index $\ell>0$ in the sense that, $\forall \lambda>0$
\begin{equation}
\label{eq:LID}
\lim_{x\to 0^+} \frac{F(\lambda x)}{F(x)} = \lambda^\ell.
\end{equation}
\end{assumption}
Again, we use a simplified notation although the value of the index $\ell$ of the regular variation is specific to the given $\q$ because it models the distribution of the distance from $\q$.
Indeed, \citet{9194069} call it the Local Intrinsic Dimension (LID) to insist that it only characterises the local neighborhood of the query point.
 
This assumption is motivated by the theory of extreme values~\citep[Chap.VIII]{Feller:1991ut}.
As $n$ increases while the ranks $\kx$ and $\kt$ are fixed, the distances of the $\kx$ and $\kt$ nearest neighbors tap into the lower tail of the distance distribution, which ought to be regularly varying because it is lower bounded by 0.
This implies that the \emph{absolute} amount of perturbation $X_n - T_n$ asymptotically goes to zero.
Yet, the asymptotical behavior of the \emph{relative} amount $\Delta_n = (X_n-T_n)/X_n$ is not trivial.
In the end, the neighborhood of $\q$ is parametrized with the single index $\ell$. 

As an example, if the neighbors are Gaussian distributed spanning an affine hyperplane of dimension $d^\prime\leq d$ and passing through query $\q$, then $\ell= d^\prime$.
The converse is of course not true. Indeed, index $\ell$ can be much smaller than the representational dimension $d$. 


\section{Statistics of the Relative Perturbation}
\label{sec:Stat}
\subsection{Finite dataset size $n$}
The joint probability distribution of $(T_n,X_n)$ has the following expression~\citep[Eq.~2.2.1]{H.-A.-David:2003ts}:\\
$\forall (t,x)\in\real^{+2}$
\begin{equation}
g_n(t,x) =
\begin{cases}
0 \quad\quad\quad\quad\quad\quad\text{if } t \geq x\geq 0 \\
\frac{F^{\kt-1}(t)[F(x) - F(t)]^{\kx-\kt - 1} [1 - F(x)]^{n-\kx}}{B(\kt,\kx-\kt)B(\kx,n-\kx+1)}\\
\quad \times f(t)f(x)\quad\text{if } x > t\geq 0,
\end{cases}
\end{equation}
where $B$ is the Beta function.

The c.d.f. $F_{\Delta_n}$ of random variable $\Delta_n$ is given by
\begin{eqnarray}
F_{\Delta_n}(\delta) &=& \Prob(\Delta_n \leq \delta) = \Prob(X_n(1-\delta) \leq T_n) \\
&=&\int_0^{+\infty} \left(\int_{(1-\delta)x}^{x} g_n(t,x) dt\right) dx.
\label{eq:IntJoint}
\end{eqnarray}

Figure~\ref{fig:Integral} illustrates function $g_n$ and the domain of integration.
Appendix~\ref{app:A} shows the following results:
\begin{proposition}
\label{prop:CDFn}
The c.d.f. of the minimal relative amount of perturbation is given by:
\begin{equation}
F_{\Delta_n}(\delta) = 1 - \Exp \left[I_{\nicefrac{F((1-\delta)F^{-1}(\Xi))}{\Xi}}(\kt,\kx-\kt)\right].
\label{eq:CDFn}
\end{equation}
with $\Xi \sim Beta(\kx,n-\kx+1)$
\end{proposition}
$I_x(\alpha,\beta)$ is the regularized incomplete beta function, \ie\ the c.d.f. of
the beta distribution $Beta(\alpha,\beta)$.

\begin{figure}[bh]
\begin{center}
\includegraphics[width=\rr\columnwidth]{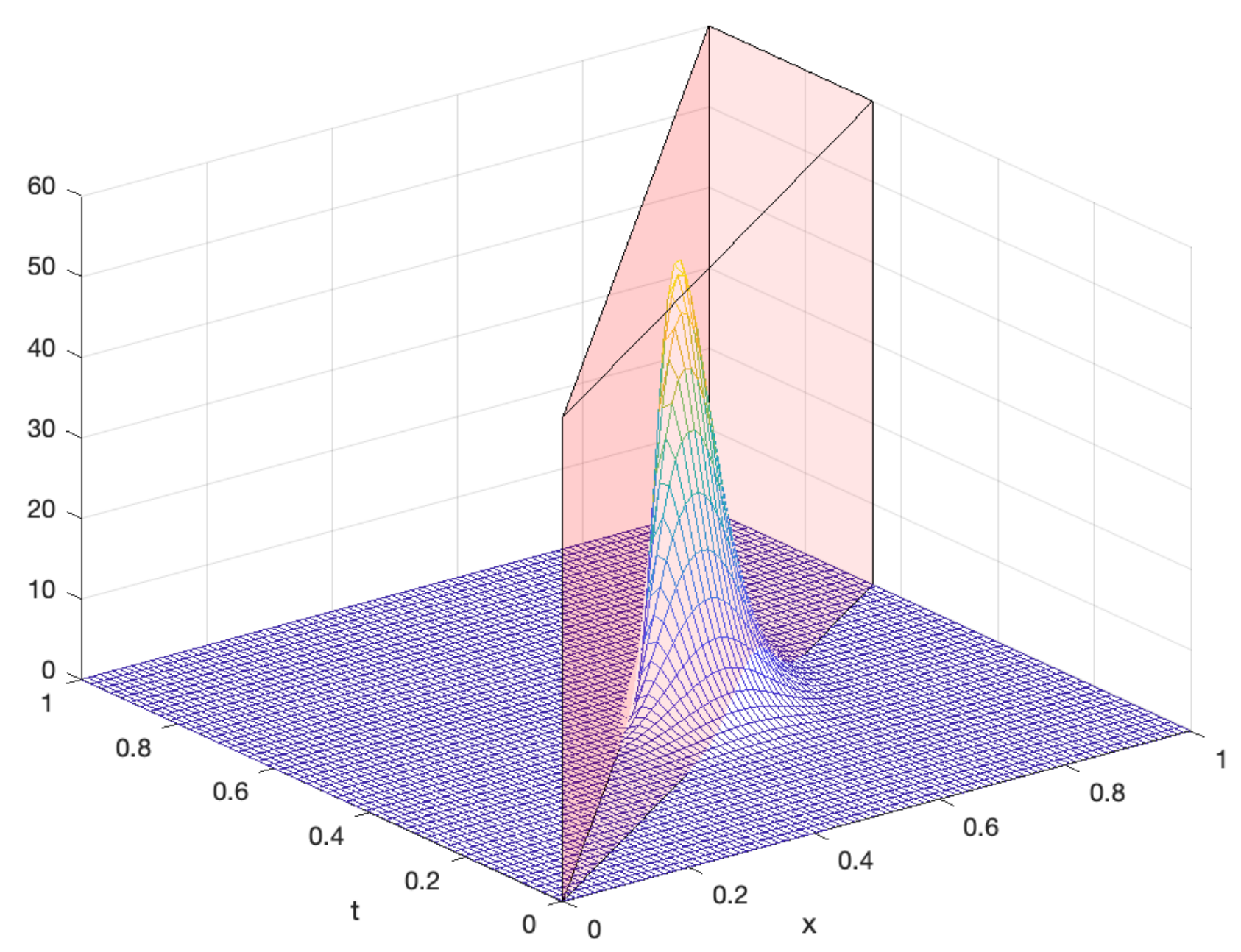}
\caption{Visualization of the part of $g_n(t,x)$ to be integrated in~\eqref{eq:IntJoint}, for $\delta=\nicefrac{1}{3}$, $\kt = 2$, $\kx=4$, $n=40$, and $F(x)=x^5$.}
\label{fig:Integral}
\end{center}
\end{figure}

\subsection{Infinite dataset size $n\to+\infty$}
We now look for a better understanding of the distribution of $\Delta_n$ asymptotically as $n\to+\infty$.
Figure~\ref{fig:Converge} shows that its c.d.f. $F_{\Delta_n}$ seems to converge to an asymptotic function.
Appendix~\ref{app:B} uses the Laplace's method to find the limit
\begin{equation}
F_{\Delta}(\delta):= \lim_{n\to+\infty} F_{\Delta_n}(\delta) = I_{L(\delta)}(\kx-\kt,\kt).
\label{eq:CDFasym}
\end{equation}
with $L(x):=1-(1-x)^\ell$.

This can be restated as the main result of this paper:
\begin{proposition}
\label{prop:CDFasy}
For two natural integers $\kx > \kt$, the relative amount of perturbation $\Delta_n$ converges in distribution to $\Delta = 1 - B^{\nicefrac{1}{\ell}}$ with $B\sim Beta(\kt,\kx-\kt)$ as the size $n$ of the dataset increases.
\end{proposition}
\begin{proof}
Since function $L$ is a strictly increasing function, $F_\Delta(\delta)$ can be expressed as
the c.d.f. of the random variable $L^{-1}(B^\prime) = 1-(1-B^\prime)^{\nicefrac{1}{\ell}}$ with $B^\prime\sim B(\kx-\kt,\kt)$:
\begin{equation}
\Prob(L^{-1}(B')\leq\delta) = \Prob(B'\leq L(\delta)) = F_\Delta(\delta).
\end{equation}
 Knowing that $B = 1-B^\prime$ is distributed as $B(\kt,\kx-\kt)$, we have $\Delta = 1 - (1-B^\prime)^{\nicefrac{1}{\ell}} = 1 - B^{\nicefrac{1}{\ell}}$.
\end{proof}

\begin{figure}[hb]
\begin{center}
\includegraphics[width=\rr\columnwidth]{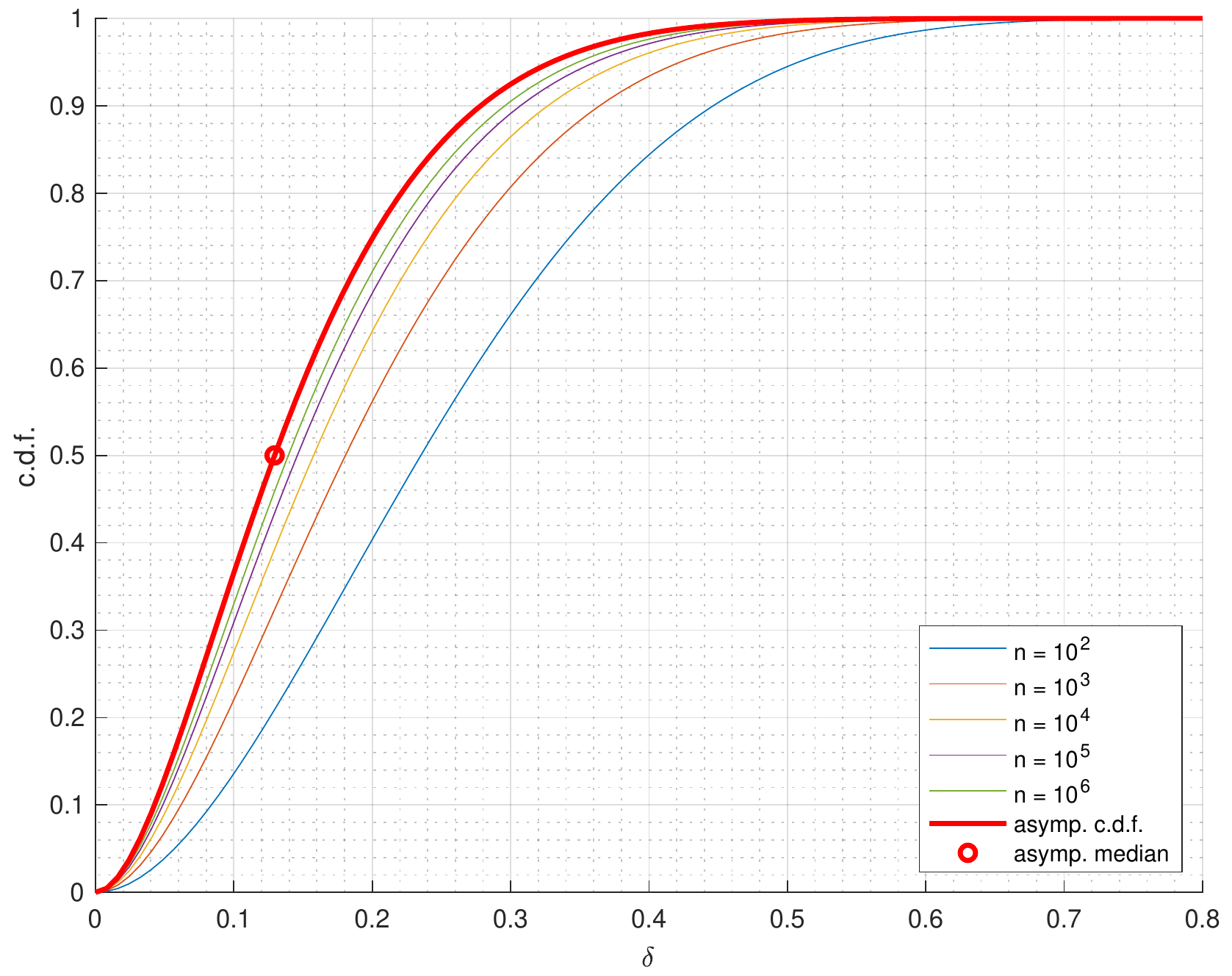}
\caption{Convergence of $F_{\Delta_n}$ to $F_\Delta$ as $n$ increases from $10^2$ to $10^6$, $\kt = 2$, $\kx=4$, and $F(x)=x^\ell$ with $\ell = 5$.}
\label{fig:Converge}
\end{center}
\end{figure}

\subsection{Statistics of Random Variable $\Delta$}

This section derives the mean, the median, and the mode associated to the asymptotical distribution of $\Delta$.
\begin{proposition}
\label{prop:StatDelta}
The r. v. $\Delta$ has the following statistics:
\begin{itemize}
\item Expectation: 
\begin{equation}
\Exp[\Delta] \approx 1 - (\nicefrac{\kt}{\kx})^{\nicefrac{1}{\ell}}\left(1 - \frac{\ell-1}{2\ell^2} \frac{\kx-\kt}{\kt(\kx+1)}\right)
\label{eq:AppExp}
\end{equation}
for $\kx\gg 1$.
\item Median:
\begin{eqnarray}
\bar{\Delta} &=&1 - (I^{-1}_{\nicefrac{1}{2}}(\kt,\kx-\kt))^{\nicefrac{1}{\ell}}\nonumber\\
&\approx& 1 - (\nicefrac{\kt}{\kx})^{\nicefrac{1}{\ell}}\left(1 + \frac{2\kt - \kx}{\kt(3\kx-2)}\right)^{\nicefrac{1}{\ell}}
\label{eq:AppMed}
\end{eqnarray}
for $\kt \geq 2$ and $\kx \geq \kt + 2$.
\item Mode:
\begin{equation}
\mathsf{mode} =1 - \left(\frac{\kt}{\kx}\right)^{\nicefrac{1}{\ell}}\left(1 + \frac{(\ell+1)\kt-\kx}{\ell\kx\kt -(\ell+1)\kt}\right)^{\nicefrac{1}{\ell}}
\label{eq:AppMode}
\end{equation}
\end{itemize}
\end{proposition}
\begin{proof}
See Appendix~\ref{app:C}, which also contains exact formulas for the special case $\kt = 1$.
\end{proof}  

These statistics are expressed in order to outline the quantity $1 - (\nicefrac{\kt}{\kx})^{\nicefrac{1}{\ell}}$, which is indeed the main result of article~\citep[Th.~2]{9194069}.
Quoting this article, it outlines that
``\textit{the amount of perturbation required to subvert neighborhood rankings diminishes with the local intrinsic dimensionality}'' $\ell$ of this neighborhood.
For instance, for large $\ell$, this term further simplifies into $\approx \log(\nicefrac{\kx}{\kt})/\ell$, which shows that $\ell$ is of utmost importance.
It has a bigger impact than the ratio $\nicefrac{\kx}{\kt}>1$ due to the log scale.

Figure~\ref{fig:PdfLid} shows how the mode of the p.d.f. of $\Delta$ decreases as $\ell$ increases, especially for large $\kx$.
It also illustrates that it becomes more peaky around its mode . 

However, \citet{9194069} do not state what this quantity is w.r.t. $\Delta$.
Their Theorem~2 states that if the relative amount of perturbation $\delta$ is strictly bigger than $1 - (\nicefrac{\kt}{\kx})^{\nicefrac{1}{\ell}} + \epsilon$,
then $\x$ has been modified into a point $\y$ whose rank $\ky < \kt$.
Yet, that $\epsilon$ is fuzzily defined in~\citep[Th.~2]{9194069}: ``\textit{For any sufficiently small real value $\epsilon>0$}''.
Our result contradicts this statement: That amount of perturbation achieves its goal only up to a probability which heavily depends on $\epsilon$:
\begin{equation}
\Prob(success) = F_\Delta\left(1 - (\nicefrac{\kt}{\kx})^{\nicefrac{1}{\ell}} + \epsilon\right).
\end{equation}


\begin{figure}[tbh]
\begin{center}
\includegraphics[width=\rr\columnwidth]{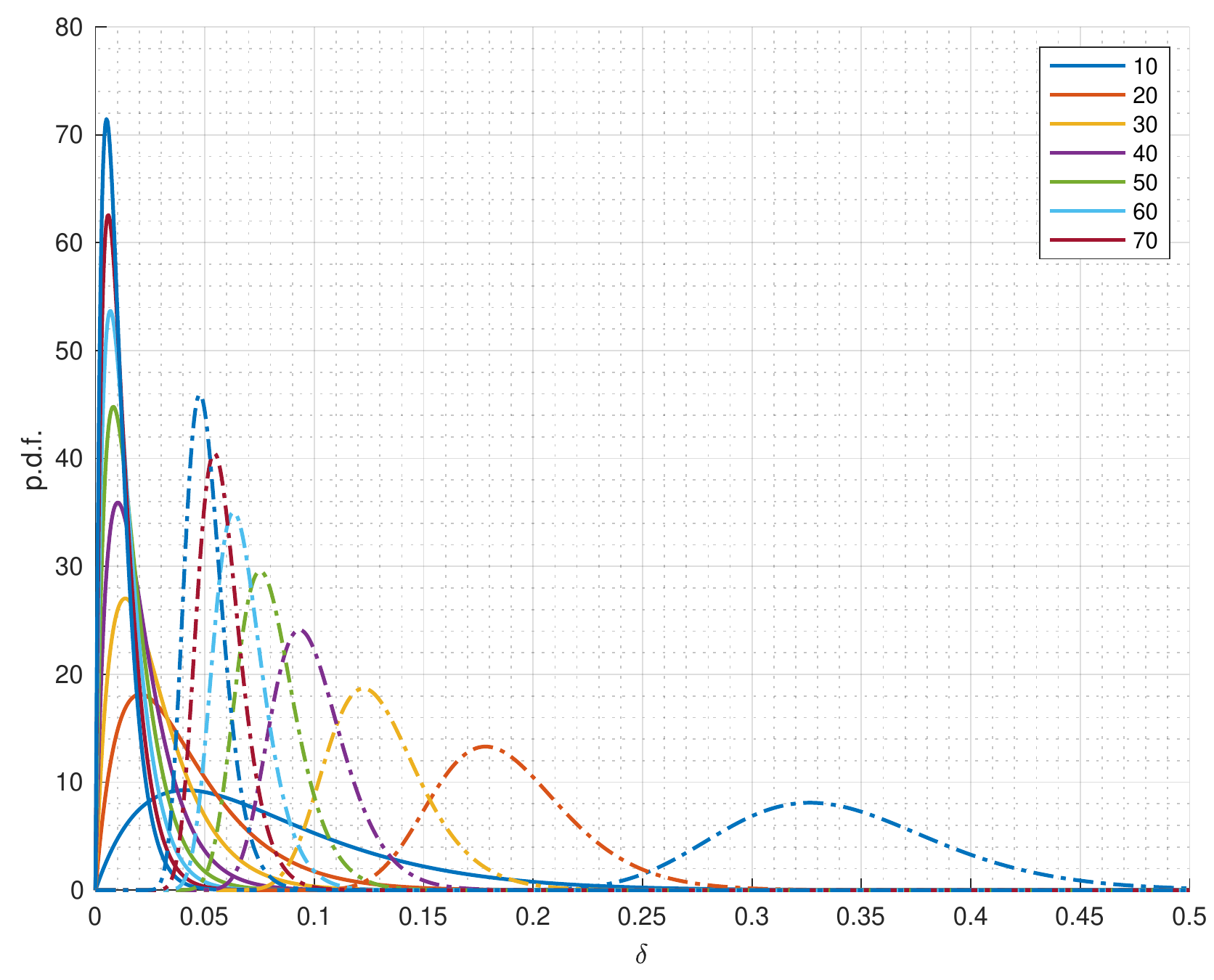}
\caption{Asymptotic probability density function of $\Delta$ as a function of $\delta\in[0,0.5]$ for values of $\ell\in\{10,20, \ldots, 70\}$ for $\kx=4$ (or $\kx = 100$, dashed) and $\kt=2$.}
\label{fig:PdfLid}
\end{center}
\end{figure}


\section{Experimental Work}
\label{sec:Experiment}

\subsection{Experimental Protocol}
\def \bigann {\texttt{BigANN}}
\def \imagenet {\texttt{ImageNet}}
\def \cifar {\texttt{CIFAR}}
\def \turing {\texttt{MSTuring}}
\def \spacev {\texttt{MSSPACEV}}
\def \deep {\texttt{DEEP}}
\def \nq{q}

The experimental protocol uses six datasets: \imagenet, \cifar, \bigann, \deep, \spacev, \turing.

The first four datasets are related to computer vision.
They correspond to global feature vectors extracted from images thanks to a deep neural network (ResNet50 for \imagenet\ and \cifar, GoogLeNet for \deep~\citep{Yandex-Deep}) or to 
local hand-crafted descriptors (SIFT for \bigann~\citep{bigann}).
\turing\ and \spacev~\citep{spacev} are related to web search. They are both released by Microsoft Research and consist of vector encoding of documents and/or queries better capturing the intent of the search.

\imagenet, \cifar, and \bigann\ are used in the article~\citep{9194069}, whereas \deep, \spacev, and \turing\ are used in the NeurIPS'21 Approximate Nearest Neighbor Search Challenge\footnote{\url{https://big-ann-benchmarks.com}}. 
Table~\ref{tab:datasets} summarizes these data ; Fig.~\ref{fig:PDF_LID} shows the empirical p.d.f. of the indices $\ell$ measured according to the following protocol. 

For each query, an exhaustive search finds the $\kx$-nearest neighbors in the dataset of size $n$, and their distances $\{d_i\}_{i=1}^{\kx}$ to the query.
The minimum relative amount of perturbation is computed by  $\delta = 1-\nicefrac{t}{x}$ where $t=d_{\kt}$ and $x=d_{\kx}$ are the distances of the $\kt$-th and $\kx$-th nearest neighbors.
For each query, the index $\ell$, so-called Local Intrinsic Dimension, is estimated using the maximum likelihood estimator, a.k.a. the Hill estimator~\citep{10.1214/aos/1176343247}:
\begin{equation}
\label{eq:Hill}
\hat{\ell} = -\left(\kx^{-1}\sum_{i=1}^{\kx} \log(\nicefrac{d_i}{d_{\kx}})\right)^{-1}
\end{equation}

The protocol provides $\nq$ pairs of measurements $\{(\hat{\ell}_i, \delta_i)\}_{i=1}^\nq$ (one per query).
We consider the neighborhoods with similar estimation values $\hat{\ell}_i$
as independent realizations of our statistical model described by assumptions~\ref{ass:1} and~\ref{ass:2}.

Figure~\ref{fig:PDF_LID} shows that the datasets have different range of indices.
The index $\ell$ takes values often bigger than 30 in \cifar\ and \imagenet, whereas \deep, \spacev\ and \turing\ are peculiar
with index values often below 5. Note that these ranges are weakly related to the space dimension (see Tab.~\ref{tab:datasets}).

\begin{table}[t]
\caption{Details about the datasets}
\begin{center}
\begin{tabular}{lcccc}
\toprule
dataset & $d$ & type & $(n,q)$ &  $\kx$ \\
\midrule
\imagenet & 2048 & float & (1,2M,10K) & 1K \\
\cifar & 2048 & float & (50K,10K) & 1K  \\
\bigann & 128 & uint8 & (1B,10K) & 100  \\
\deep & 96 & float & (1B,10K) & 100  \\
\turing & 100 & float & (1B,100K) & 100 \\
\spacev & 100 & uint8 & (1B,29K) & 100 \\
\bottomrule
\end{tabular}
\end{center}
\label{tab:datasets}
\end{table}%

\begin{figure}[t]
\begin{center}
\includegraphics[width = \rr\columnwidth]{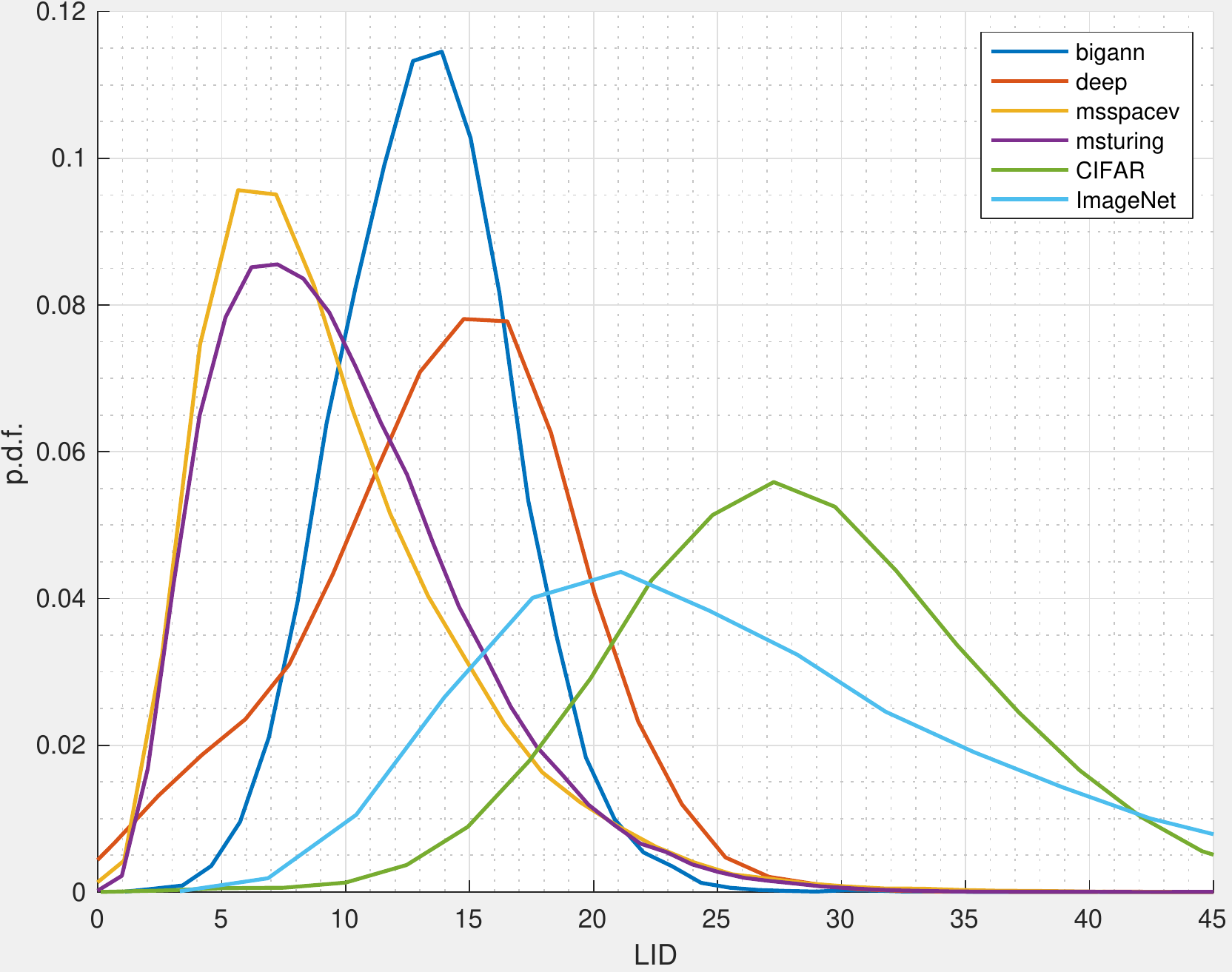}
\caption{Empirical p.d.f. of Hill estimations $\hat{\ell}$~\eqref{eq:Hill} with the kernel density estimator of~\citet{1337238}.}
\label{fig:PDF_LID}
\end{center}
\end{figure}

\subsection{Statistical Analysis}
The kernel density estimator of~\citet{10.1214/10-AOS799} learns an empirical joint p.d.f. $\pi(\ell,\delta)$ defined over $\real^+\times [0,1]$ from the measurements.
In the same way, we learn an empirical marginal p.d.f. $\hat{f}_L(\ell)$ from the estimations $\{\hat{\ell}_i\}_{i=1}^\nq$ (3rd cardinal spline KDE with Silverman rule~\citep{1337238}). 
Our goal is to compare the empirical joint p.d.f. $\pi(\ell,\delta)$ to the one based on our model, $f_\Delta(\delta|\ell)\hat{f}_L(\ell)$, coming from~\eqref{eq:PDFasym}.
Figure~\ref{fig:ImageNetPDF} shows the very good match between the empirical and joint models for ImageNet.
The same observation holds for the other datasets (see figures in Appendix~\ref{app:ExpRes}).

For further investigation, we would like to compare the empirical marginal and the asymptotical p.d.f. of $\Delta$.
To gather a maximum of samples, we normalize  the LID index $\ell$ with the following transformation:
\begin{proposition}
\label{prop:Normalized}
Suppose that $\Delta$ follows the asymptotical distribution given in Prop.~\ref{prop:CDFasy} with index $\ell$ and the ranks $(\kt,\kx)$,
then $\tilde{\Delta} = 1 - (1 - \Delta)^{\nicefrac{\ell}{\ell_0}}$ also follows the asymptotical distribution with the same ranks, but index $\ell_0$.
\end{proposition}
\begin{proof}
Since $B = (1 - \Delta)^\ell \sim Beta(\kt,\kx-\kt)$, then $\tilde{\Delta} = 1 - (1 - \Delta)^{\nicefrac{\ell}{\ell_0}} = 1 - B^{\nicefrac{1}{\ell_0}}$.
\end{proof}

We apply this trick to create the samples $\{\tilde{\delta}_i\}_{i=1}^{\nq}$ with one common unique index value $\ell_0 = 10$.
It is as if all the measurement were coming from independent neighborhoods realizing the statistical model with index $\ell_0$.
The empirical p.d.f. of $\tilde{\Delta}$ is learned with the method of~\citet{1337238}.
Figure~\ref{fig:PlotTrick} compares this p.d.f. with the asymptotical model $f_\Delta(\cdot|\ell_0)$ in~\eqref{eq:PDFasym} from different pairs $(\kt,\kx)$.
One can see the very close fitting to the theoretical model although the datasets are of different nature, size, and setup.

Such a good fitting is surprising because we only made two assumptions about the data, and moreover, they are violated by some datasets.
For instance, \spacev\ and \bigann\ contains quantized vectors belonging to $\{0,1,\ldots,255\}^{d}$.
Therefore, the distances between points take discrete values and this is a clear violation of assumption~\ref{ass:2}.
The model is indeed robust enough to tackle a fine enough granularity of the distance values.

\begin{figure}[bt]
\begin{center}
\includegraphics[width = \rr\columnwidth]{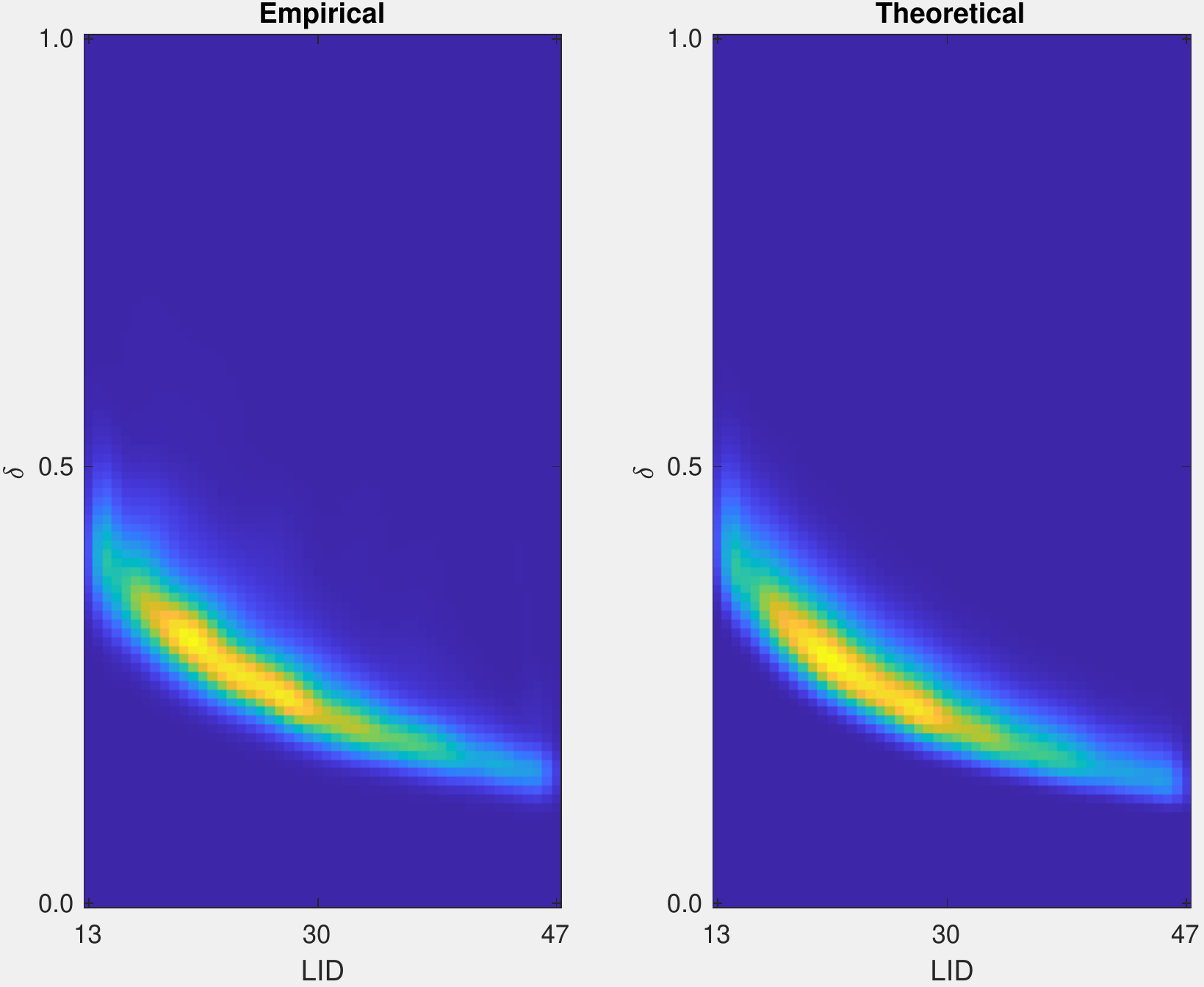}
\caption{
Comparison of the empirical (learned with k.d.e.~\citep{1337238}) and theoretical p.d.f. of the normalized r.v. $\tilde{\Delta}$ defined in Prop.~\ref{prop:Normalized} with $\ell_0 = 10$, $\kx = 1000$, $\kt=1$.}
\label{fig:ImageNetPDF}
\end{center}
\end{figure}

\begin{figure}[h]
\begin{center}
\includegraphics[width=\rr\columnwidth]{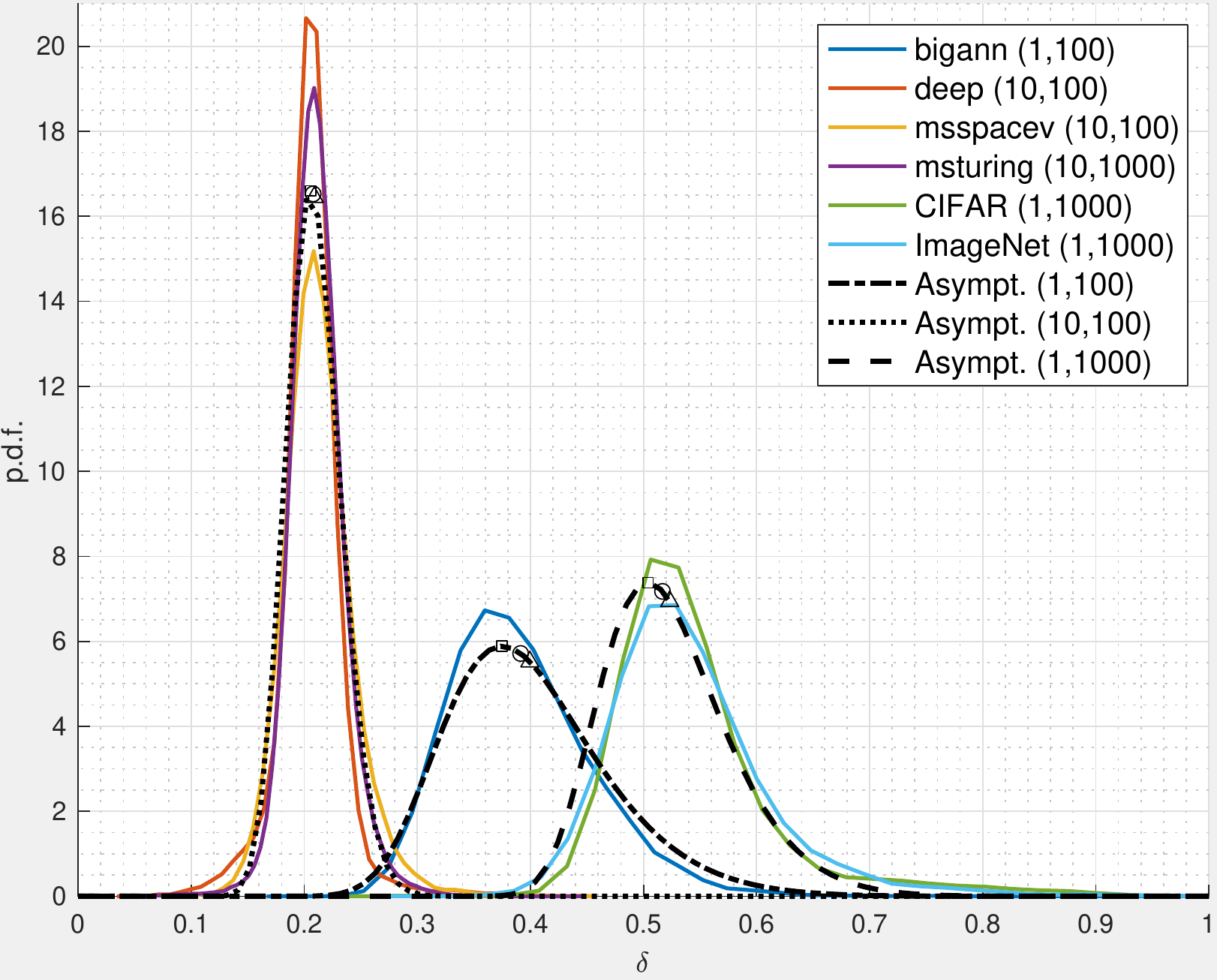}
\caption{Comparison of the empirical p.d.f. of the normalized data $\tilde{\Delta}$ and their related asymptotical models (see Prop.~\ref{prop:Normalized}) with $\ell_0=10$ and $(\kt,\kx)\in\{(1,100), (10,100), (1,1000)\}$. Expectation ($\triangle$), median ($\circ$), and mode ($\square$).}
\label{fig:PlotTrick}
\end{center}
\end{figure}

\subsection{Discrepancies and Explanations}
This section presents the limits of the proposed model.
It first explains the very light discrepancies visible in Fig.~\ref{fig:PlotTrick}, and then exhibits typical scenarios where the model fails.
 
\subsubsection{Convergence Rate}
In Fig.~\ref{fig:PlotTrick}, \cifar\ and \imagenet\ p.d.f. shows a small shift of $\approx 0.03$ compared to the theoretical curve.
This is due to the fact that they do not yet reach the asymptotical regime as their size is not so large (see Tab.~\ref{tab:datasets}).
Indeed, this offset is already visible in Fig.~\ref{fig:Converge} even for dataset size as large as one million.

For the billion scale datasets of Tab.~\ref{tab:datasets}, Fig.~\ref{fig:KSConverge} shows the Kolmogorov-Smirnov goodness of fit test value
between the empirical and the asymptotical c.d.f. of $\tilde{\Delta}$ when downsizing the collection of points by a factor 10 or 100.
This illustrates the slow convergence to the asymptotical model as $n$ increases.

\begin{figure}[tb]
\begin{center}
\includegraphics[width=\rr\columnwidth]{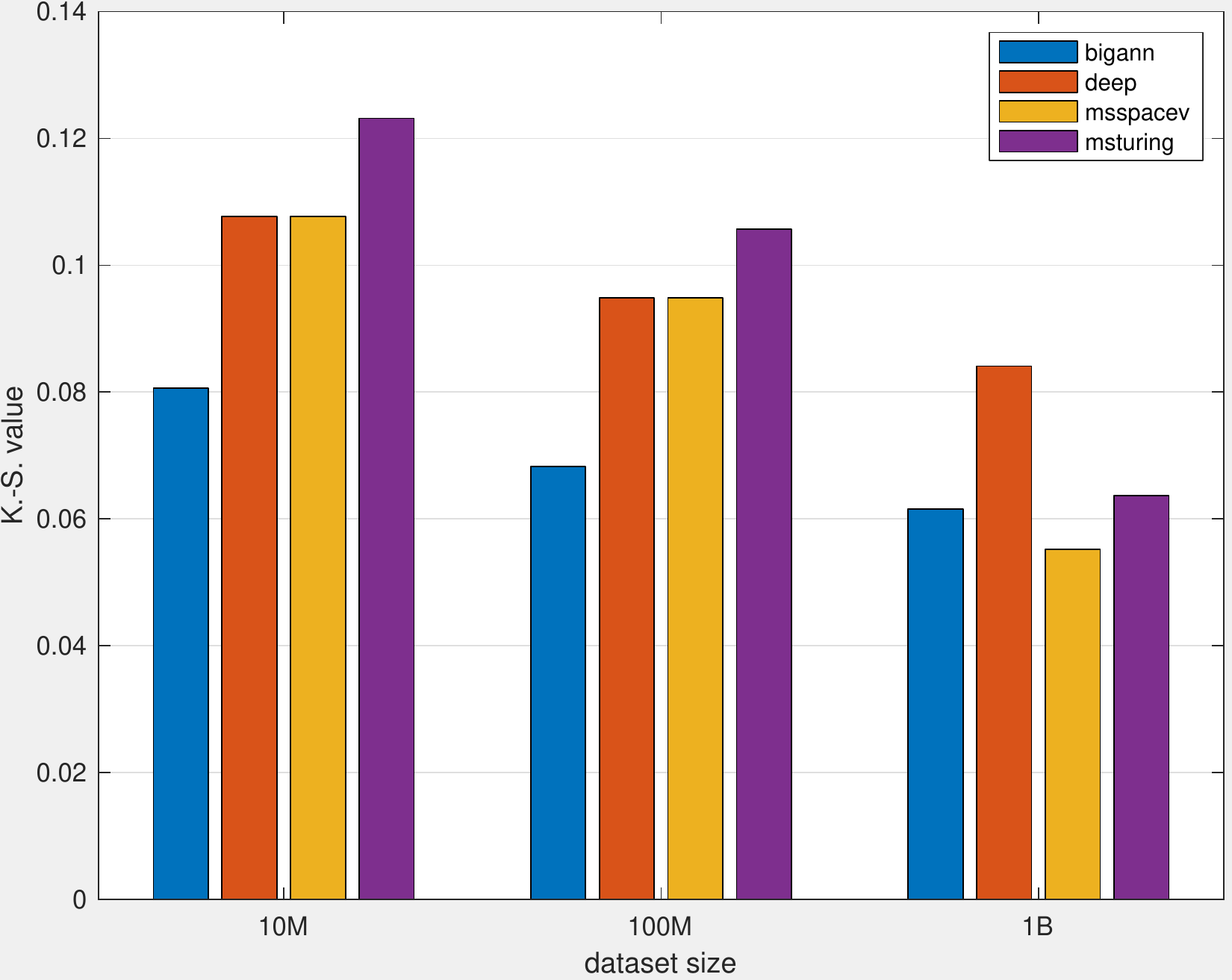}
\caption{Kolmogorov-Smirnov goodness of fit test value w.r.t. to the size of the dataset. Same setup as in Fig.~\ref{fig:PlotTrick}.}
\label{fig:KSConverge}
\end{center}
\end{figure}

\subsubsection{Inaccurate Estimation of the Index $\ell$}
All the empirical p.d.f.'s have a slightly different skewness than the theoretical plot.
We do not have access to the index $\ell$, we can only estimate its value from the distances of the $\kx$ nearest neighbors.
The measurement $\delta$ is normalized into $\tilde{\delta}$ by the estimated index  $\hat{\ell}$.
Even if $\delta$ were occurences of r.v. $\Delta$, $\tilde{\delta}$ does not follow the distribution of Prop.~\ref{prop:Normalized} due to the estimation noise present in $\ell$.

This discrepancy is almost not visible for \cifar, \imagenet, and \bigann.
For these datasets, the estimated index is rarely below 5 as shown in Fig.~\ref{fig:PDF_LID}.
We evaluate the inaccuracy of the Hill estimator by the Kolmogorov-Smirnov goodness of fit test value for typical estimations $\hat{\ell}$.
Figure~\ref{fig:MeanKS} shows that indeed estimating indices $\ell<10$ is difficult. This is especially true for \deep\ and may explain its heavier tail in Fig.~\ref{fig:PlotTrick}. 
The joint p.d.f. displayed in Appendix~\ref{app:ExpRes} confirm this phenomenon: for all datasets, the differences between empirical and theoretical are more visible at lower index value (\ie\ on the left side of the heat map representation).

\begin{figure}[htb]
\begin{center}
\includegraphics[width=\rr\columnwidth]{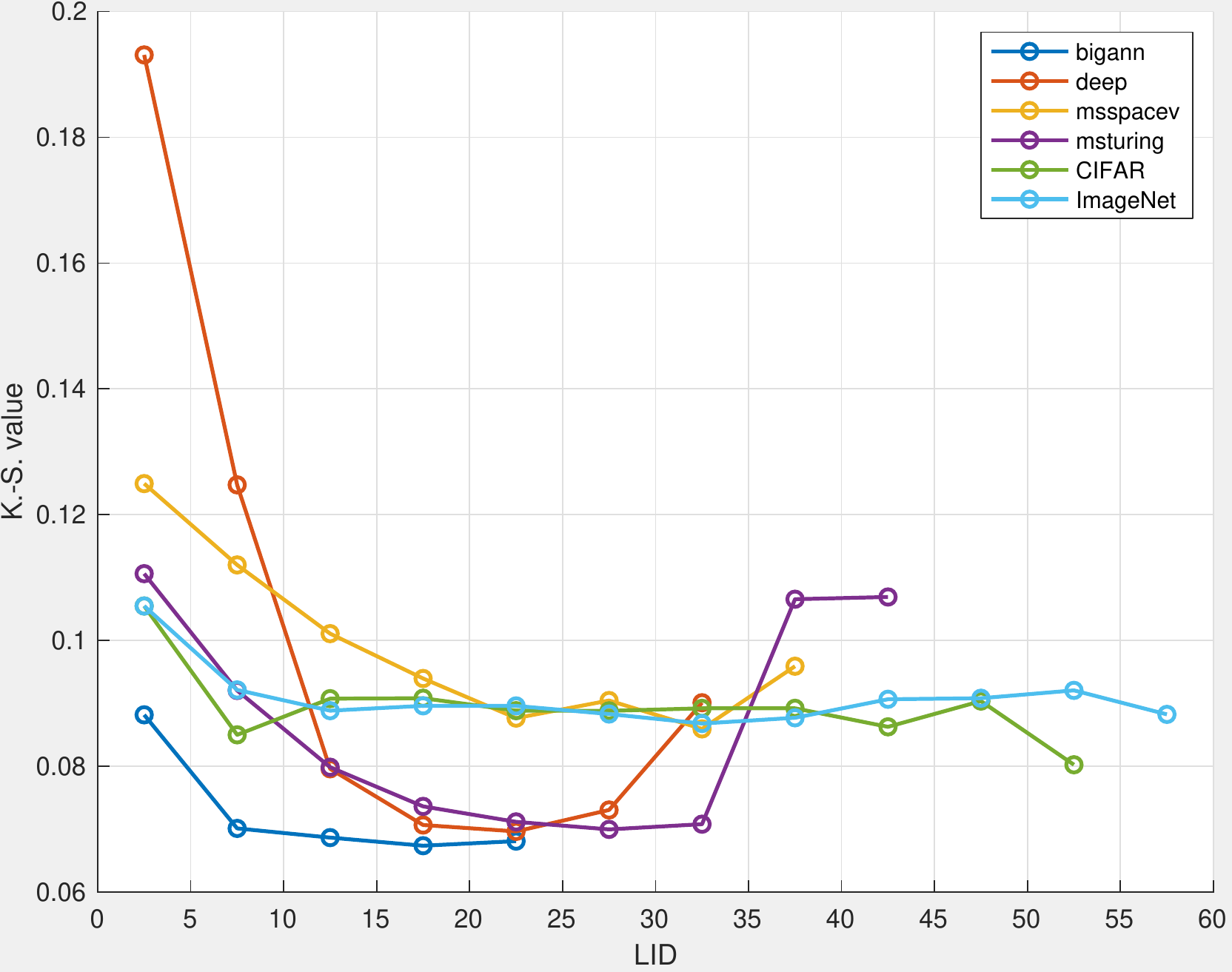}
\caption{Empirical average of Kolmogorov-Smirnov test values as a function of $\hat{\ell}$. The estimated indices are binned into intervals $5\times[k,k+1)$ with $k\in\{0,1,\ldots, 11\}.$}
\label{fig:MeanKS}
\end{center}
\end{figure}

\subsubsection{Mixture of distributions}
Figure~\ref{fig:Failures} shows typical failures of the model.
We use some of the earlier datasets, \spacev\, \turing, and \deep\ but on a specific setup $(\kt,\kx)=(1,100)$.
Note that this setup was not an issue for \bigann in Fig.~\ref{fig:PlotTrick}.

Assumption~\ref{ass:1} supposes that the distances from the query follow a unique distribution.
This might not be the case for some queries. We envisage the following cases:
The $\kx=100$ smallest distances may correspond to:
\begin{enumerate}
\item a subset of documents  matching the query,
\item a subset of documents not matching the query,
\item the union of two matching / non matching subsets.
\end{enumerate}
Cases 1 and 2 comply with assumption~\ref{ass:1}, but not case 3.

\spacev\ and \turing\ are collections of data related to web search with vector encodings of queries and documents.
It is expected that a web search ends up with returned documents ranked first because they have a semantic connection to the query (\ie a match) or because it statically happens that their distances are small. 
Case 3 might also happen with \deep\ which is a collection of features extracted from ``\textit{images of the web}'', dixit~\citet{Yandex-Deep}.
Over such a huge collection of pictures scrapped from the internet, it is likely that there exist duplicates or quasi-copy of some query images in the dataset, yielding case 3.

This is a priori not the case with \cifar\ and \imagenet.
For these last two datasets, the feature vectors are extracted with a classifier deep neural network and the images are coming from a collection meant for classification.
For instance, if the query is an image of a dog, it will match the hundreds of dog images in the collection.
Therefore, case 1 is likely and this does not violate our assumption.
These are our suppositions, there is no way to certify occurrences of case 3.

\def \km {k_m}
We now explain with a hand-waving rationale that the mixture of distributions explains the heavy tails in Fig.~\ref{fig:Failures}.
Suppose for instance that the $\km$ nearest neighbors are matches following a different distribution than the $\kx-\km$ followers, with $\kt\leq \km<\kx$.
Measurement $\hat{\ell}$ may be corrupted by this mixture of distributions, but since it is essentially the inverse of an average over $\kx$ distances~\eqref{eq:Hill},
the estimate is robust provided that $\km\ll\kx$.
On the contrary, measurement $\delta$ is very sensitive to distance $\dt$~\eqref{eq:Equivalence}.
Therefore, if $\kt <\km$, distance $\dt=d_{\kt}$ is expected to be significantly smaller than the others, and $\delta$ gets close to 1.
Moreover, a value $\delta\lesssim 1$ is indeed spread over a wide upper interval due to the normalization mapping tuned with a very low ratio $\nicefrac{\hat{\ell}}{\ell_0}$.
Fig.~\ref{fig:MeanKS} shows that \deep, \spacev, and \turing\ have queries with low estimates $\hat{ell}$.
This amplifies the discrepancy and explains the heavy upper tail.
This means that, under case 3, the perturbation is bigger than statistically predicted by the model.

The fact that these failures do not happen in Fig.~\ref{fig:PlotTrick} where $\kt = 10$, means that $1< \km< 10$ in \deep, \spacev, and \turing.
Indeed, the smaller $\kt$ the more likely a case 3 occurs.

\begin{figure}[ht]
\begin{center}
\includegraphics[width=\rr\columnwidth]{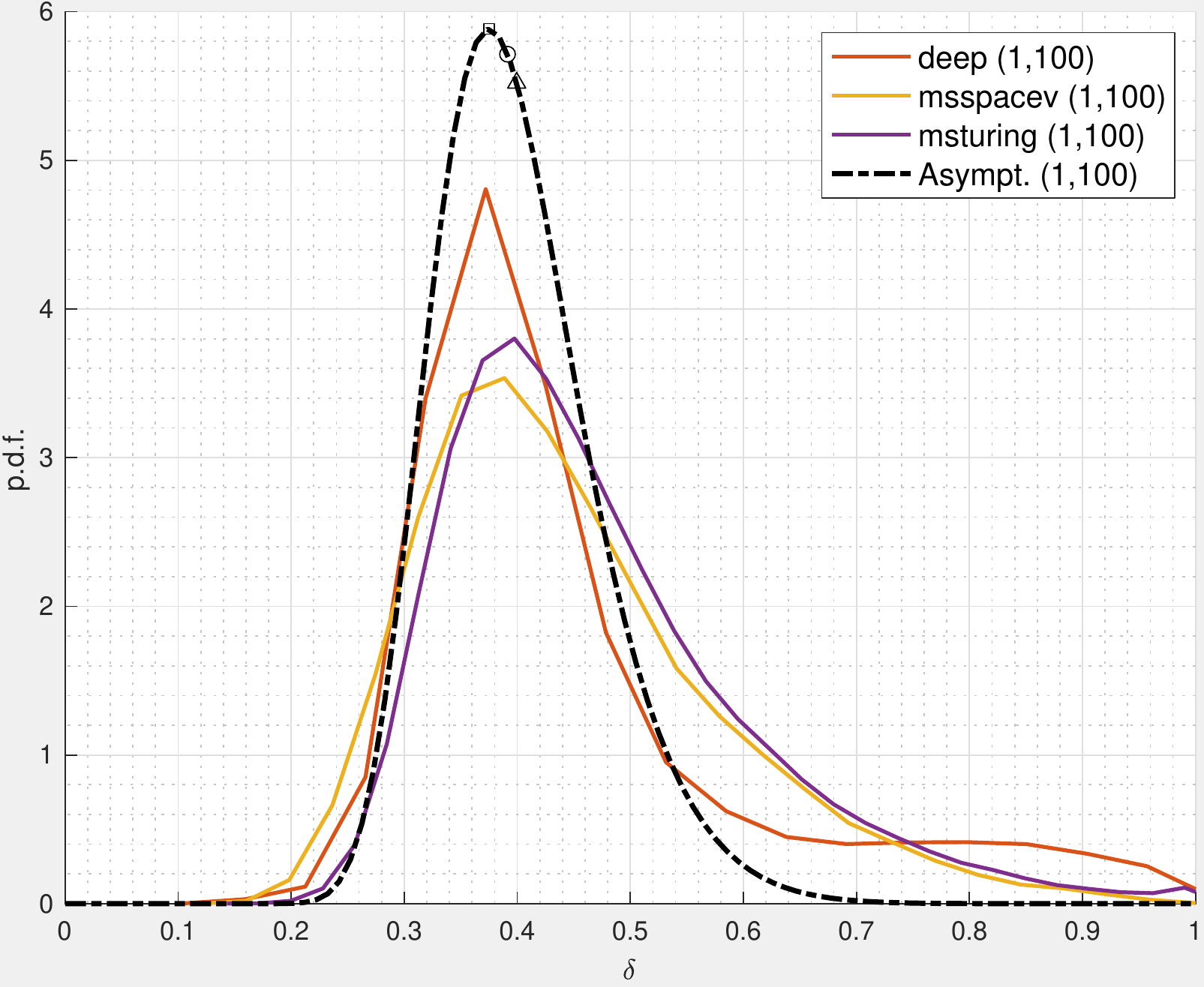}
\caption{Cases of failure. Comparison of the empirical p.d.f. of the normalized data $\tilde{\Delta}$ and their related asymptotical models (see Prop.~\ref{prop:Normalized}) with $\ell_0=10$ and $(\kt,\kx)=(1,100)$. Expectation ($\triangle$), median ($\circ$), and mode ($\square$).}
\label{fig:Failures}
\end{center}
\end{figure}

\section{Discussion}
\label{sec:Discussion}

As in~\cite{9194069}, the analysis can be extended to two cases:
\begin{itemize}
\item The modified point goes away from a reference point in the sense that the target rank $\kt$ is bigger than the actual rank $\kx$.
\item It is the query $\q$ that is perturbed and the dataset points are untouched. The query is pushed closer to (or away from) a target dataset point so that its rank decreases (resp. increases).
\end{itemize}

The second point is easy to handle by re-defining $\Delta_n = \nicefrac{T_n}{X_n} - 1$ with $T_n\geq X_n$.
\begin{proposition}
\label{prop:CDFasyAway}
For two natural integers $\kx < \kt$, the relative amount of perturbation $\Delta_n$ converges in distribution to $\Delta = B^{-\nicefrac{1}{\ell}}-1$ with $B\sim Beta(\kx,\kt-\kx)$, as the size $n$ of the dataset increases.
\end{proposition}
\begin{proof}
From~Prop.~\ref{prop:CDFasy}, we have that the ratio of two neighbor distances converges to $B^{\nicefrac{1}{\ell}}$ with $B$ following a Beta distribution.
Here,  we have $X_n < T_n$, so that $(X_n /T_n)^\ell\leq1$ converges to $Beta(\kx,\kt-\kx)$ and $\Delta_n$ converges to $\Delta = B^{-\nicefrac{1}{\ell}}-1$
\end{proof}

The second point applies as well provided that the index $\ell$ now characterizes the distances from the target dataset point.
As $n$ increases, the query and this dataset point get closer (for a fixed rank $\kx$) and~\citet{9194069}  argue that they asymptotically share the same index value.
In other words, the index $\ell$ can be considered as a very smooth continuous function over the location in the space.

As said in the introduction, the vulnerability of retrieval also depends on the ability to forge a content, say an image, whose extracted feature vector exactly corresponds to $\y$.
This point is out of our scope and \citet{DBLP:journals/corr/SabourCFF15} and \citet{tolias2019targeted} give experimental evidence as far as description with deep neural networks is concerned. If one knows the Lipschitz constant $C$ of the neural network, then $C^{-1}\Delta$ gives a statistical lower bound of the Euclidean distortion in the input space.
Yet, such an image may not exist because the description may not be a surjection.
Targeting point $\y$ reduces the distortion in the feature space but not necessarily in the input space.

\section{Conclusion}
The article~\cite{9194069} gives asymptotical statistics about the relative amount of perturbation outlining its dependence on the Local Intrinsic Dimensionality.  
This new proof refines this result by providing the exact expression of its cumulative distribution function for fixed $n$~\eqref{eq:CDFn} and its convergence to the limit distribution~\eqref{eq:CDFasym}. 
Expectation, median, and mode show similar dependence on the Local Intrinsic Dimensionality index $\ell$.
The relative amount of necessary perturbation not only decreases in expectation but also concentrates around this expectation when  
$\ell$ increases for a fixed rank ratio $\nicefrac{\kt}{\kx}$, or for an increasing rank ratio and a fixed index $\ell$.
 

\bibliography{Bib}

\newpage
\appendix
\onecolumn
\section{Proof of Proposition~\ref{prop:CDFn}}
\label{app:A}

The c.d.f. $F_{\Delta_n}$ of r.v. $\Delta_n$ is given by:
\begin{equation}
F_{\Delta_n}(\delta) =
\begin{cases}
&0, \quad\forall \delta \leq 0\\
&1, \quad\forall \delta \geq 1,
\end{cases}
\end{equation}
and $\forall \delta \in(0,1)$:
\begin{eqnarray}
F_{\Delta_n}(\delta) &=& \Prob(\Delta_n \leq \delta) = \Prob(X_n(1-\delta) \leq T_n) \\
&=&\int_0^{+\infty} \int_{(1-\delta)x}^{\infty} g_n(t,x) dt.dx\\
&=&\frac{1}{B(\kx,n-\kx+1)}\int_0^{+\infty} H(\delta,x) f(x) [1 - F(x)]^{n-\kx} dx,
\end{eqnarray}
with
\begin{eqnarray}
H(\delta,x)&:=&\frac{1}{B(\kt,\kx-\kt)}\int_{(1-\delta)x}^{x} F^{\kt-1}(t)f(t)[F(x) - F(t)]^{\kx-\kt - 1} dt\nonumber\\
&=& F(x)^{\kx-1}\int_{\frac{F((1-\delta)x)}{F(x)}}^{1} \frac{\tau^{\kt-1} [1-\tau]^{\kx-\kt-1}}{B(\kt,\kx-\kt)}d\tau\nonumber\\
&=& F(x)^{\kx-1} \left(1 - I_{\nicefrac{F((1-\delta)x)}{F(x)}}(\kt,\kx-\kt)\right),
\end{eqnarray}
where the change of variable is defined by $\tau := F(t)/F(x)$ and $I_x(\alpha,\beta)$ is the regularized incomplete beta function.
In the end, with another change of variable $\xi:=F(x)$
\begin{eqnarray}
F_{\Delta_n}(\delta) &=& \frac{1}{B(\kx,n-\kx+1)}\int_0^{1} H(\delta,F^{-1}(\xi)) (1 - \xi)^{n-\kx} d\xi \\
&=& \int_0^{1} \left(1 - I_{\nicefrac{F((1-\delta)F^{-1}(\xi))}{\xi}}(\kt,\kx-\kt)\right) \frac{\xi^{\kx-1} (1 - \xi)^{n-\kx}}{B(\kx,n-\kx+1)} d\xi\label{eq:PrevLaplace}\\
&=& 1 - \Exp\left[I_{\nicefrac{F((1-\delta)F^{-1}(\Xi))}{\Xi}}(\kt,\kx-\kt)\right],
\end{eqnarray}
with $\Xi \sim Beta(\kx,n-\kx+1)$.

\section{Proof of Proposition~\ref{prop:CDFasy}}
\label{app:B}
We use the following version of Laplace's method stated in~\cite{Bonnaillie-Noel:2004aa}:
\begin{lemma}
\label{lemma:Laplace}
Let $g,h$ : $(0,1)\to\real$ two continuous functions s.t.
\begin{itemize}
\item $\int_0^1 |g(\xi)|e^{h(\xi)}d\xi < +\infty$
\item $\exists \delta_0>0, \forall \delta\in(0,\delta_0), \forall \xi\in[\delta,1), h(x)< h(\delta)$
\item Around $0^+$,
\begin{eqnarray}
g(\xi) &\approx& A\xi^\alpha, \quad\text{with }\alpha >-1 \\
h(\xi) &\approx& a - c\xi^{\beta} + o(\xi^\beta), \quad\text{with } c>0, \beta > 0
\end{eqnarray}
\end{itemize}
Then
\begin{equation}
\int_0^1 g(\xi)e^{nh(\xi)} d\xi \approx \frac{A}{\beta}\Gamma\left(\frac{\alpha+1}{\beta}\right)e^{an}(cn)^{-\frac{\alpha+1}{\beta}}\quad \mbox{as } n\to +\infty.
\end{equation}
\end{lemma}


We then rewrite \eqref{eq:PrevLaplace} in the form $\int_0^1 g(\xi)e^{nh(\xi)} dx / B(\kx,n-\kx+1)$ with
\begin{eqnarray}
h(\xi) &=& \log(1-\xi)\\
g(\xi) &=& \left(1 - I_{\nicefrac{F((1-\delta)F^{-1}(\xi))}{\xi}}(\kt,\kx-\kt)\right) \xi^{\kx-1} (1 - \xi)^{-\kx},
\end{eqnarray}
Thanks to the assumption~\eqref{eq:LID}, we know that
\begin{equation}
 \lim_{\xi\to 0^+} \frac{F((1-\delta)F^{-1}(\xi))}{\xi} = \lim_{x\to 0^+} \frac{F((1-\delta)x)}{F(x)} = (1-\delta)^\ell.
\end{equation}

This allows to instantiate the constants defined in Lemma~\ref{lemma:Laplace} as
\begin{eqnarray}
A = I_{1-(1-\delta)^\ell}(\kx-\kt,\kt), \quad \alpha &=& \kx -1,\nonumber\\
a = 0, \quad c = 1, \quad \beta &=& 1.
\end{eqnarray}
This leads to the following asymptotical expression:
\begin{equation}
F_{\Delta_n}(\delta) \approx I_{1-(1-\delta)^\ell}(\kx-\kt,\kt) \frac{n!}{n^{\kx} (n-\kx)!}\quad \mbox{as } n\to +\infty,
\end{equation}
and to the limit function $F_\Delta$:\\
$\forall \delta\in(0,1)$
\begin{equation}
F_{\Delta}(\delta):= \lim_{n\to+\infty} F_{\Delta_n}(\delta) = I_{1-(1-\delta)^\ell}(\kx-\kt,\kt),
\end{equation}
thanks to the well-known asymptotic convergence $\Gamma(n+\alpha)\approx\Gamma(n)n^\alpha$, $\forall\alpha\in\mathbb{C}$.

\section{Proof of proposition~\ref{prop:StatDelta}}
\label{app:C}
This appendix uses the fact that $\Delta = 1 - B^{\nicefrac{1}{\ell}}$ with $B\sim Beta(\kt,\kx-\kx)$.
\subsection{Expectation}
We did not find any close-form expression for $\Exp[\Delta]$.
Yet, the expectation and the variance of $B\sim Beta(\kt,\kx-\kx)$ equal
\begin{equation}
\Exp[B] = \frac{\kx-\kt}{\kx}, \quad\Var[B] = \frac{\Exp[B](1-\Exp[B])}{(\kx+1)}.
\end{equation}
We see that $B$ concentrates around its expectation as $\kx$ becomes large.
The second order Taylor series around $\Exp[B]$ gives
\begin{eqnarray}
\Delta &=& 1 - B^{\nicefrac{1}{\ell}}\\
&=& A_0 + A_1(B-\Exp[B]) + A_2(B-\Exp[B])^2 + o((B-\Exp[B])^2)\quad\text{with}\\
A_0 &=& 1 - \Exp[B]^{\nicefrac{1}{\ell}}, \quad A_1= - \nicefrac{1}{\ell}\Exp[B]^{\nicefrac{1}{\ell}-1}, \quad A_2 = - \nicefrac{1}{2\ell}(\nicefrac{1}{\ell}-1)\Exp[B]^{\nicefrac{1}{\ell}-2}.
\end{eqnarray}
Applying the expectation operator on both sides of the equation we obtain: for $\kx\gg1$,
\begin{eqnarray}
\Exp[\Delta] &\approx& 1-\Exp[B]^{\nicefrac{1}{\ell}} -  \frac{1}{2\ell}\left(\frac{1}{\ell} -1\right)\Exp[B]^{\nicefrac{1}{\ell}-2}\Var[B]
=1 - \left(\frac{\kt}{\kx}\right)^{\nicefrac{1}{\ell}}\left(1 - \frac{\ell-1}{2\ell^2} \frac{\kx-\kt}{\kt(\kx+1)}\right).
\end{eqnarray}

\subsection{Median}
Translating quantiles of $B$ to quantiles of $\Delta$ is easy as $x\to 1 - x^{\nicefrac{1}{\ell}}$ is a strictly decreasing function over $(0,1)$:
The $p$-quantile of $B$ equals $I^{-1}_p(\kt,\kx-\kt)$ so that the $p$-quantile of $\Delta$ equals $1 - (I^{-1}_{1-p}(\kt,\kx-\kt))^{\nicefrac{1}{\ell}}$.
For instance, the median $\bar{\Delta}$ of $\Delta$ is simply
$\bar{\Delta}  = 1-\bar{B}^{\nicefrac{1}{\ell}}$.
The median $\bar{B}$ of $Beta(\alpha,\beta)$ is usually approximated by $(\alpha-\nicefrac{1}{3})/(\alpha+\beta-2/3)$ for $\alpha\geq 2$ and $\beta\geq 2$.
This translates to the following approximation: for $\kt \geq 2$ and $\kx \geq \kt + 2$:
\begin{eqnarray}
\bar{\Delta}  &\approx& 1 - \left(\frac{\kt - \nicefrac{1}{3}}{\kx - \nicefrac{2}{3}}\right)^{\nicefrac{1}{\ell}}
= 1 - (\nicefrac{\kt}{\kx})^{\nicefrac{1}{\ell}}\left(1 + \frac{2\kt - \kx}{\kt(3\kx-2)}\right)^{\nicefrac{1}{\ell}}.
\end{eqnarray}

\subsection{Mode}
We first compute the p.d.f. by derivating the c.d.f $F_\Delta$~\eqref{eq:CDFasym}:
\begin{equation}
f_\Delta(\delta|\ell) = \frac{\ell}{B(\kx-\kt,\kt)} \left(1 - (1-\delta)^\ell\right)^{\kx-\kt-1}(1-\delta)^{\kt\ell-1}\mathds{1}_{[0,1]}(\delta).
\label{eq:PDFasym}
\end{equation}
If $\kt\ell>1$ and $\kx-\kt>1$, the derivative of the p.d.f. cancels in three points: $\delta\in\{0, \delta_m, 1\}$.
Since $f_\Delta(0) = f_\Delta(1) = 0$, and $f_\Delta(\delta)\geq 0, \, \forall \delta\in(0,1)$, this shows that
\begin{equation}
\delta_m := 1 - \left(\frac{\kt\ell-1}{(\kx-1)\ell-1}\right)^{\nicefrac{1}{\ell}} =1 - \left(\frac{\kt}{\kx}\right)^{\nicefrac{1}{\ell}}\left(1 + \frac{(\ell+1)\kt-\kx}{\ell\kx\kt -(\ell+1)\kt}\right)^{\nicefrac{1}{\ell}}
\end{equation}
is the unique maximum, hence the mode of this distribution.

\subsection{Special Case: $\kt=1$}
\def \a {\ell}
\def \b {\kx-1}
If $\kt = 1$, then $B = (1-\Delta)^\ell\sim Beta(1,\kx-1)$  and $K := B^{\nicefrac{1}{\ell}} = 1-\Delta$ follows by definition a Kumaraswamy distribution $\mathcal{K}(\a,\b)$.
This gives simpler formulas:
\begin{itemize}
\item Expectation:
\begin{equation}
\Exp[\Delta] = 1 - \Exp[B] = 1 - (\b)B(1+\nicefrac{1}{\a},\b).
\end{equation}
\item Median:
\begin{equation}
\bar{\Delta} = 1 - \bar{K} = 1 - \left(1-2^{-\frac{1}{\b}}\right)^{\nicefrac{1}{\a}}.
\end{equation}
\item Mode:
\begin{equation}
\mathsf{mode} = 1 -  \left(\frac{\a-1}{(\b)\a-1}\right)^{\nicefrac{1}{\a}}.
\end{equation}
\end{itemize}

\section{Experimental Results}
\label{app:ExpRes}
This appendix shows figures
comparing on the left the empirical joint p.d.f. $\pi(\ell,\delta)$ (learned with k.d.e.~\citep{10.1214/10-AOS799}) and its theoretical counterpart, and on the right
the comparison of the empirical (learned with k.d.e.~\citep{1337238}) and theoretical p.d.f. of the normalized r.v. $\tilde{\Delta}$ defined in Prop.~\ref{prop:Normalized} (with $\ell_0=10$)
for the datasets \bigann, \cifar, \imagenet, \deep, \spacev, and \turing\ (one per page).
Theoritical expectation ($\triangle$), median ($\circ$), and mode ($\square$) are also displayed.

\def \rr{4}

\begin{figure}[h]
\begin{center}
\includegraphics[width = 0.3\textwidth,height=\rr cm]{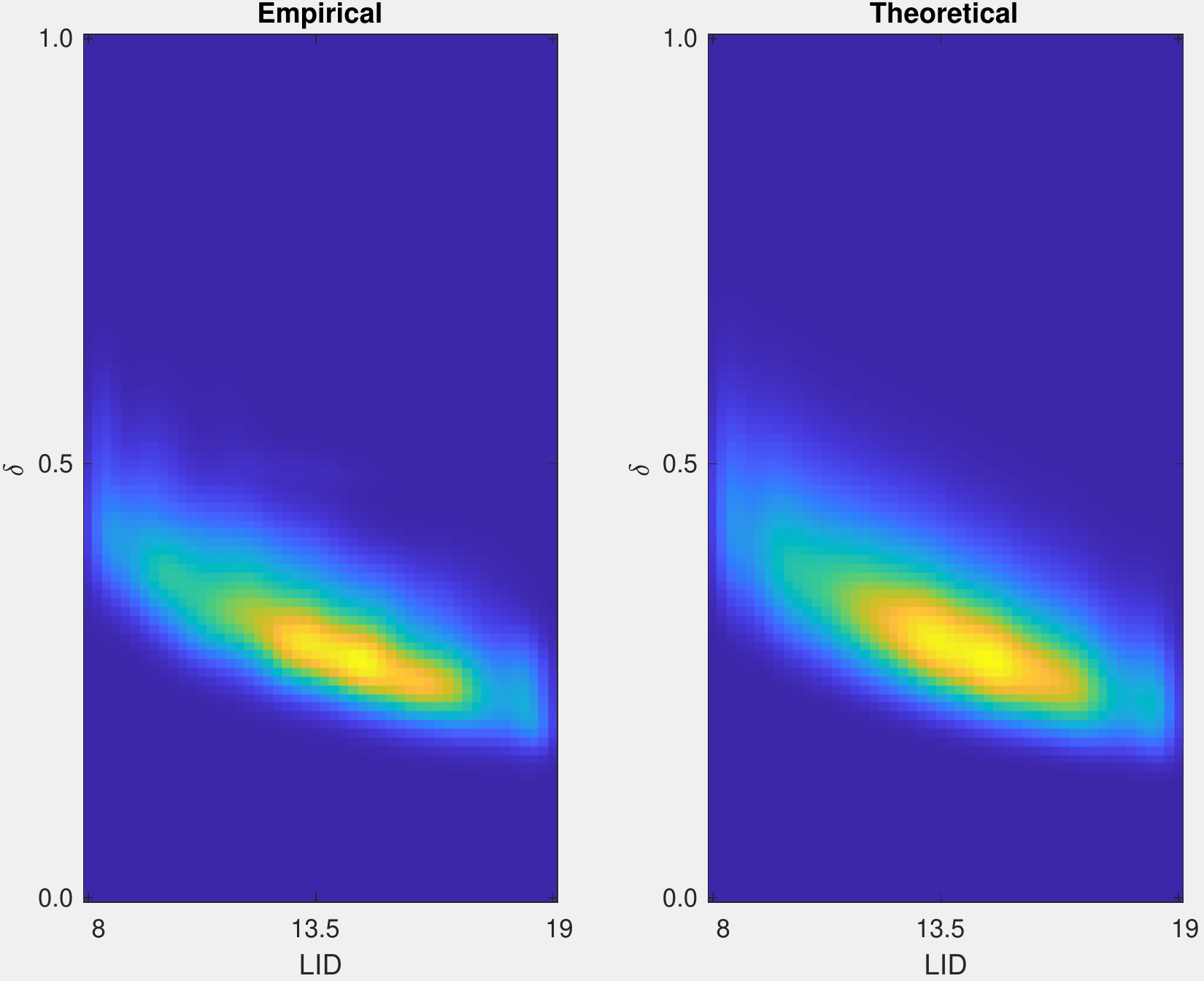}
\includegraphics[width = 0.3\textwidth,height=\rr cm]{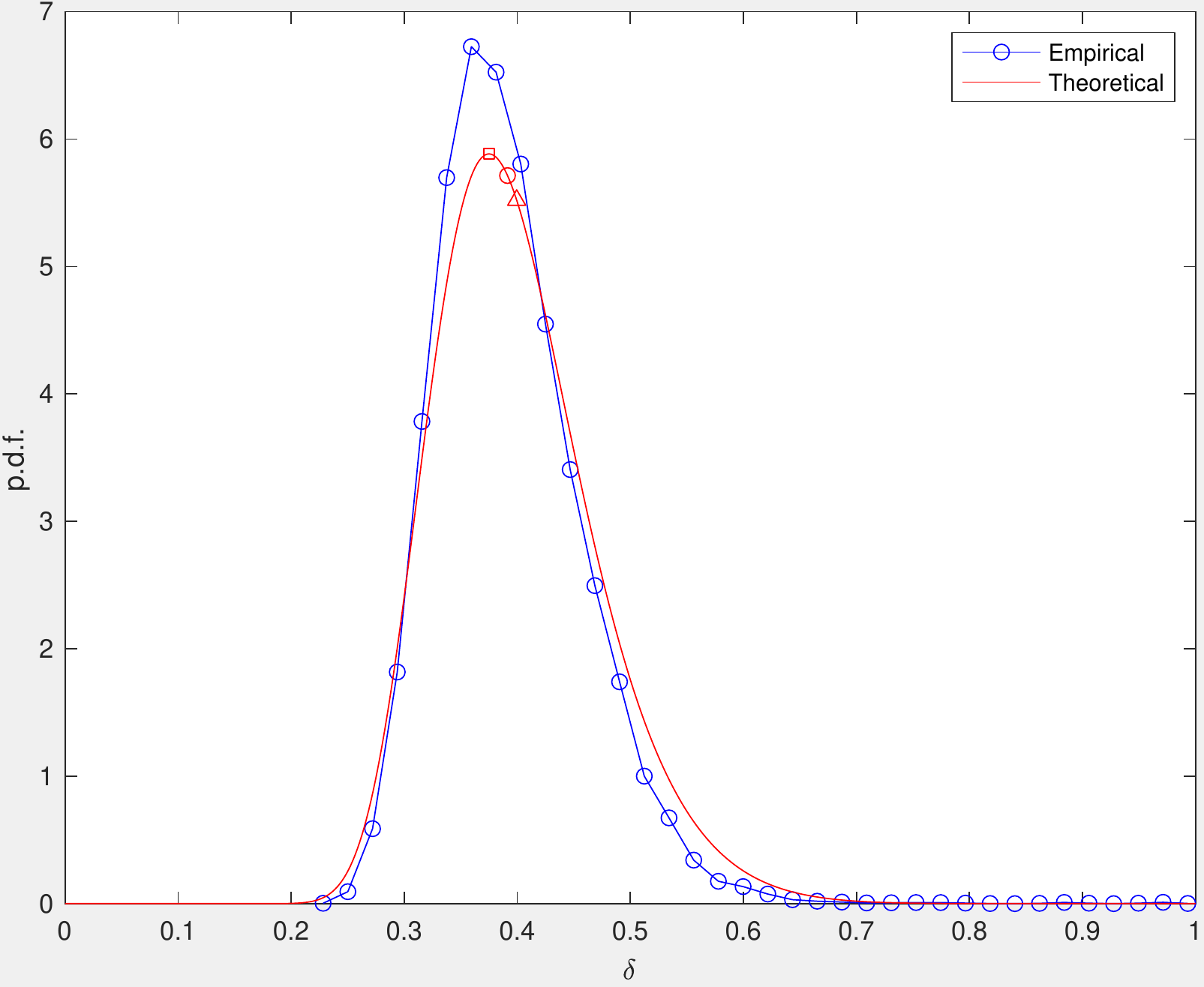}
\caption{\bigann\ with $\kx = 100$, $\kt=1$.}
\end{center}
\end{figure}

\begin{figure}[h]
\begin{center}
\includegraphics[width = 0.3\textwidth,height=\rr cm]{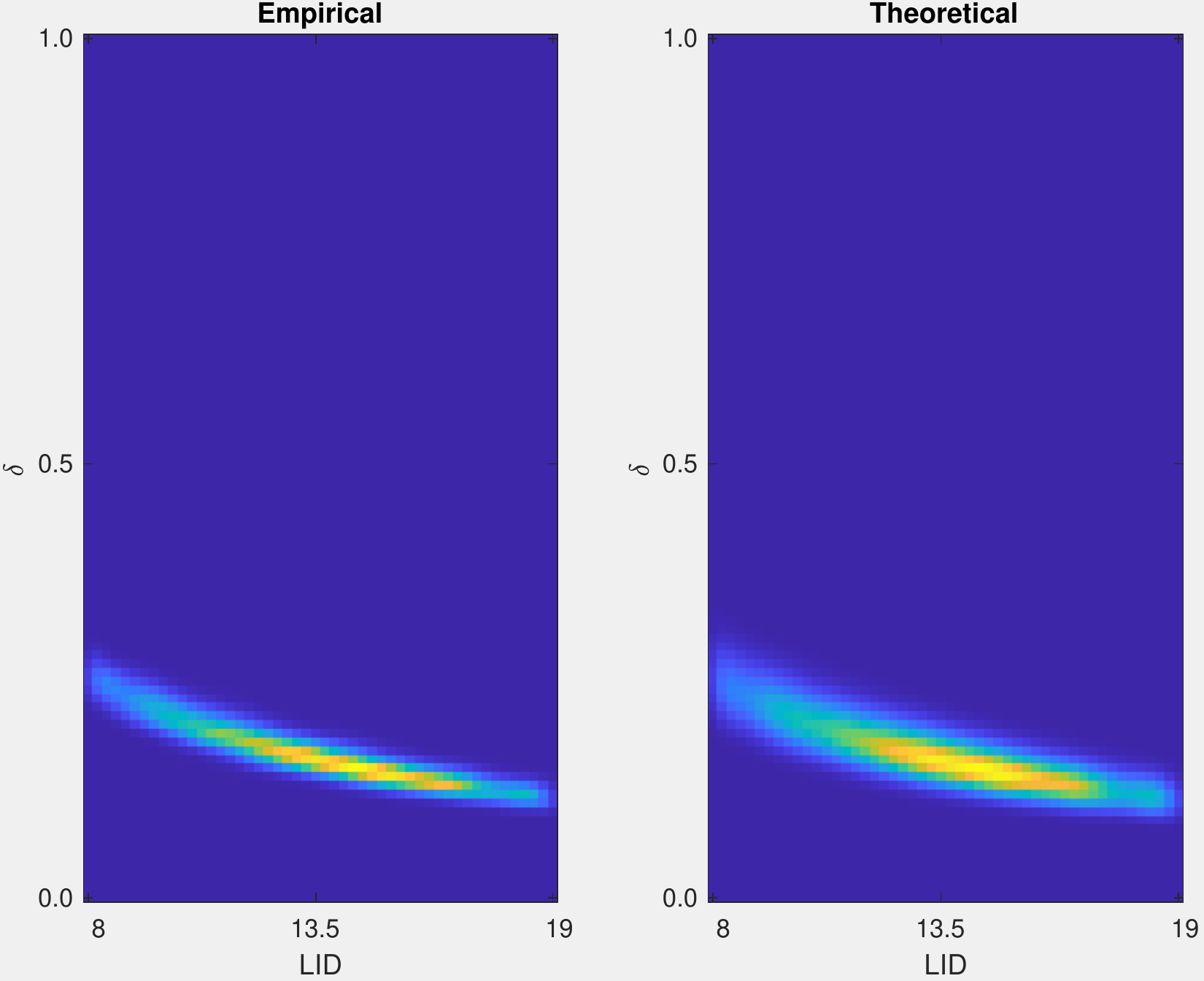}
\includegraphics[width = 0.3\textwidth,height=\rr cm]{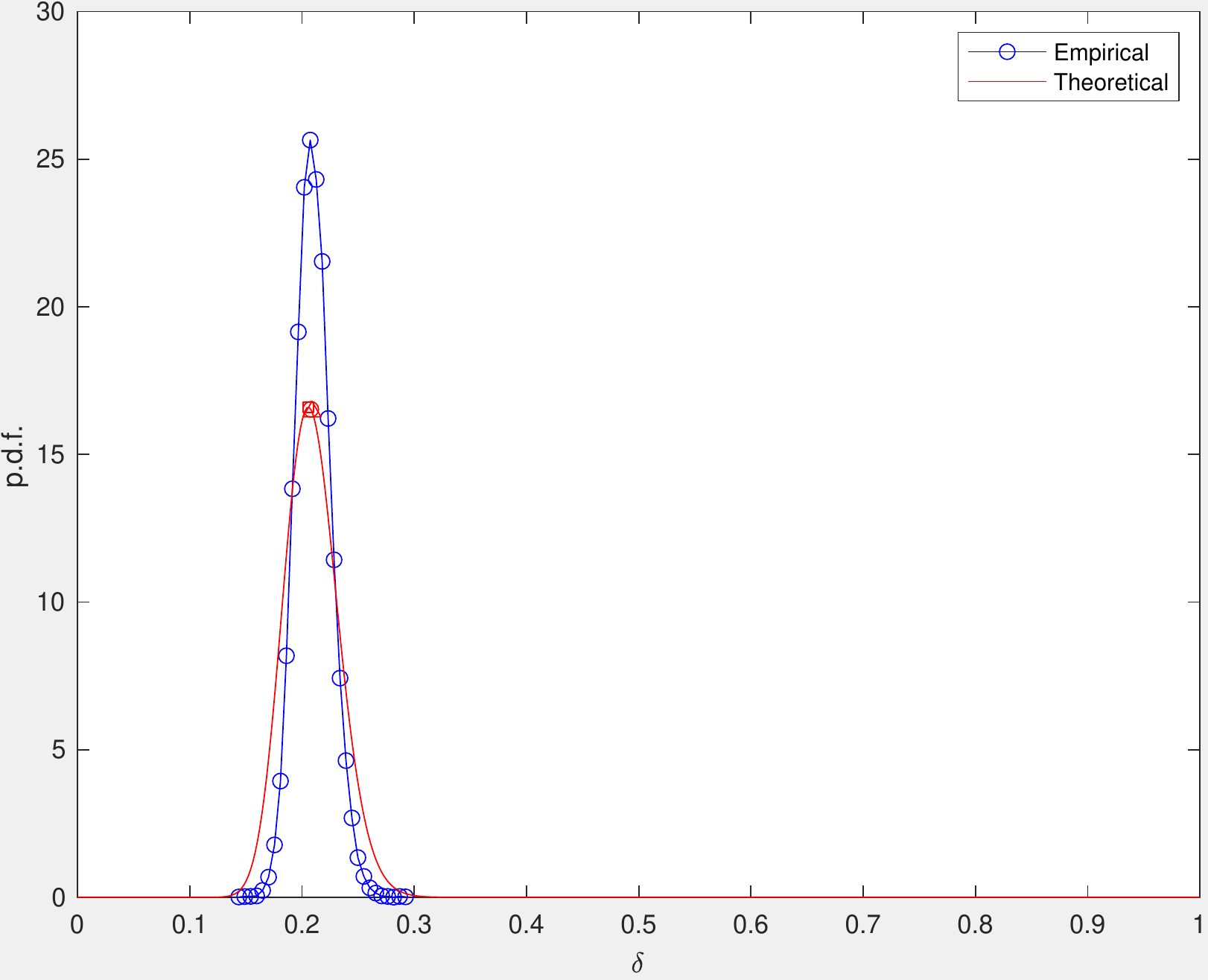}
\caption{\bigann\ with $\kx = 100$, $\kt=10$.}
\end{center}
\end{figure}

\begin{figure}[h]
\begin{center}
\includegraphics[width = 0.3\textwidth,height=\rr cm]{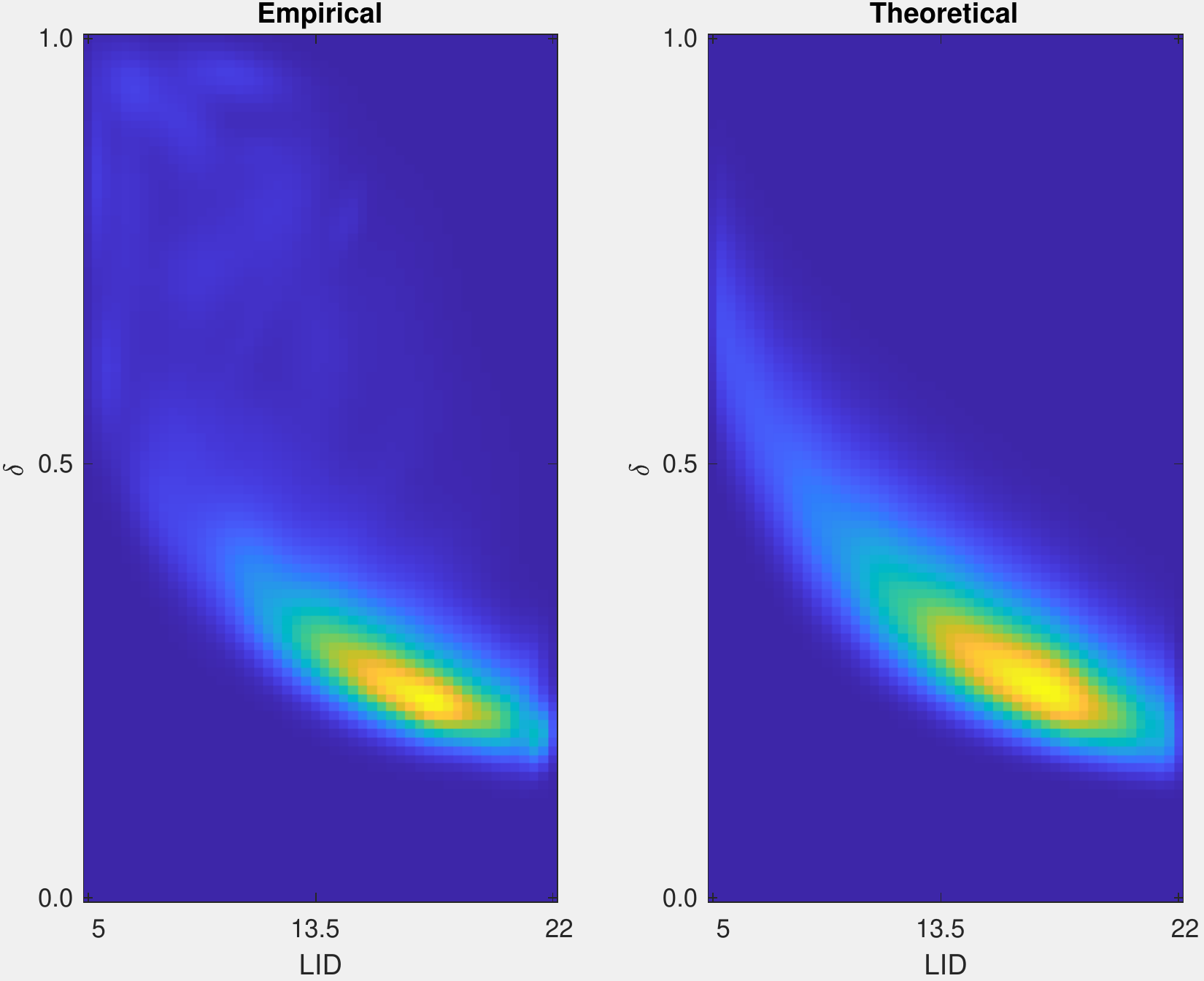}
\includegraphics[width = 0.3\textwidth,height=\rr cm]{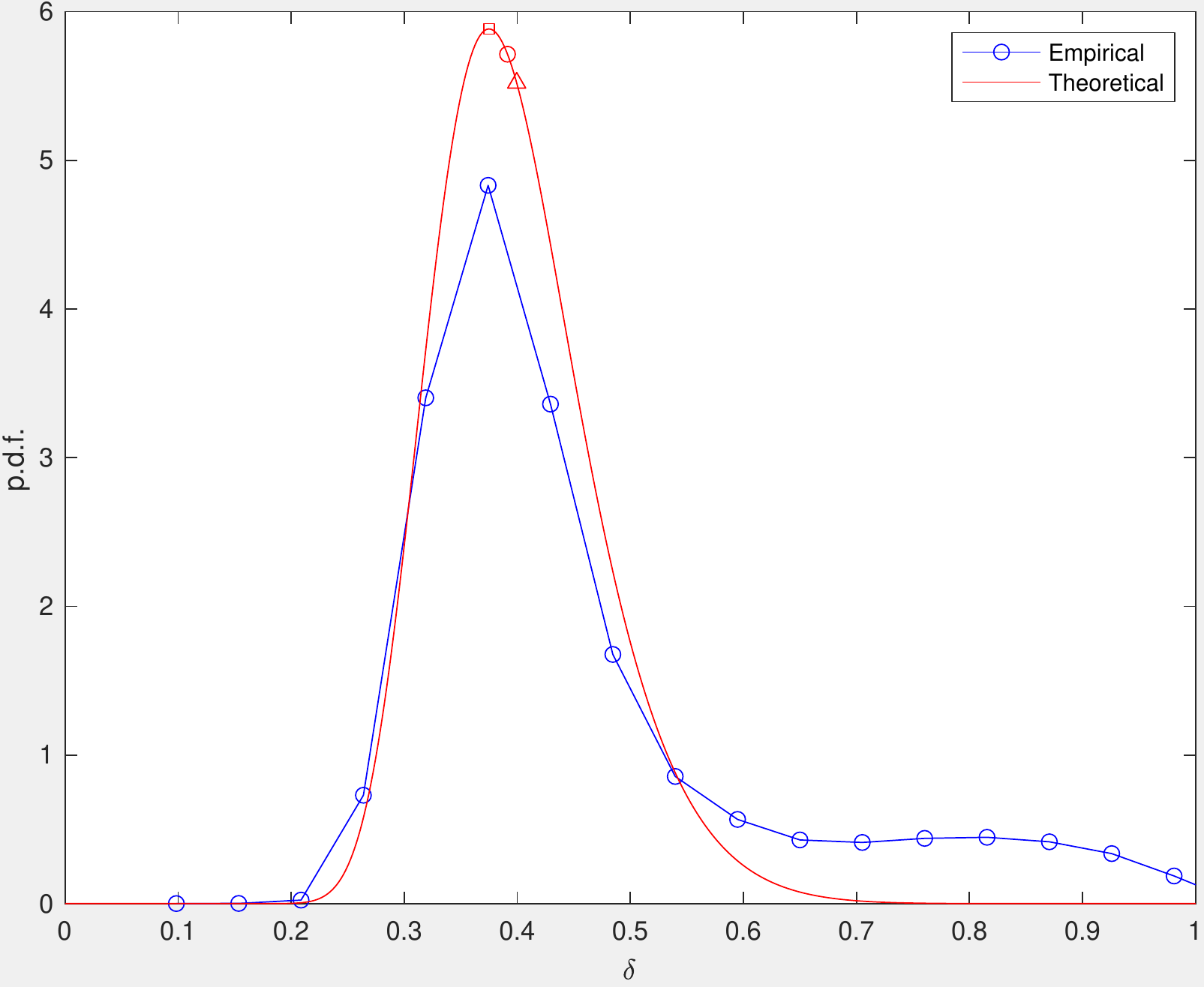}
\caption{\deep\ with $\kx = 100$, $\kt=1$.}
\end{center}
\end{figure}

\begin{figure}[h]
\begin{center}
\includegraphics[width = 0.3\textwidth,height=\rr cm]{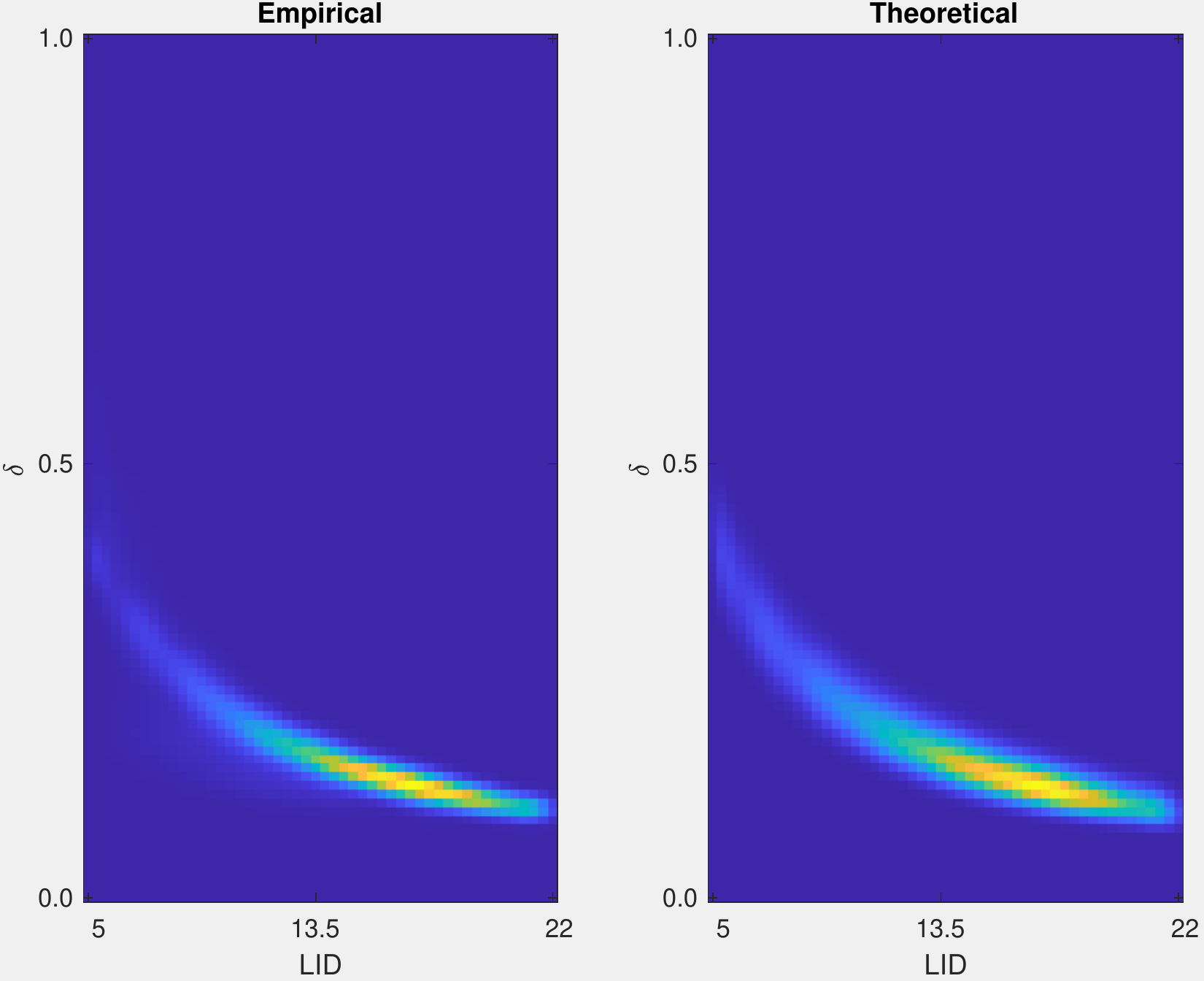}
\includegraphics[width = 0.3\textwidth,height=\rr cm]{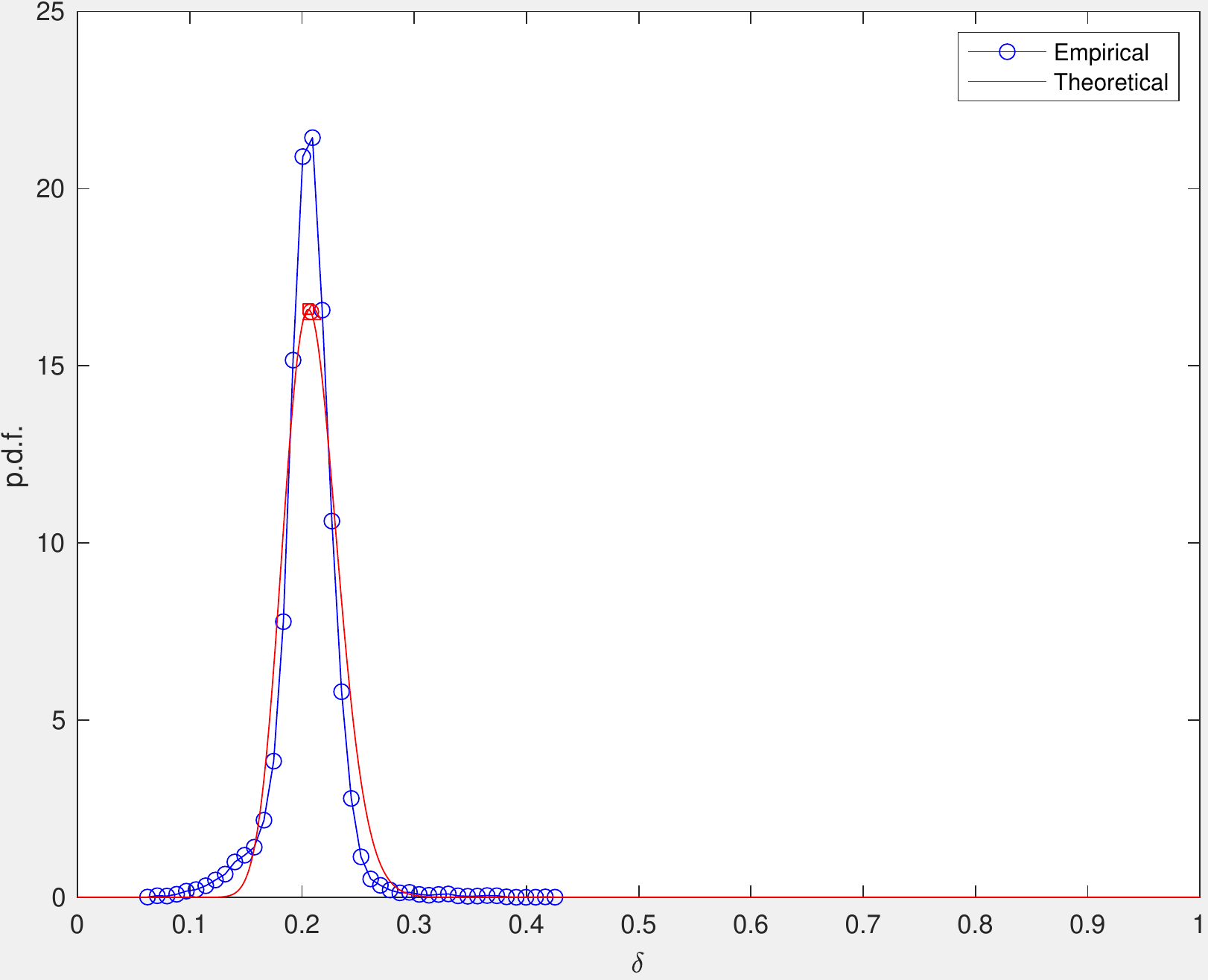}
\caption{\deep\ with $\kx = 100$, $\kt=10$.}
\end{center}
\end{figure}


\begin{figure}[h]
\begin{center}
\includegraphics[width = 0.3\textwidth,height=\rr cm]{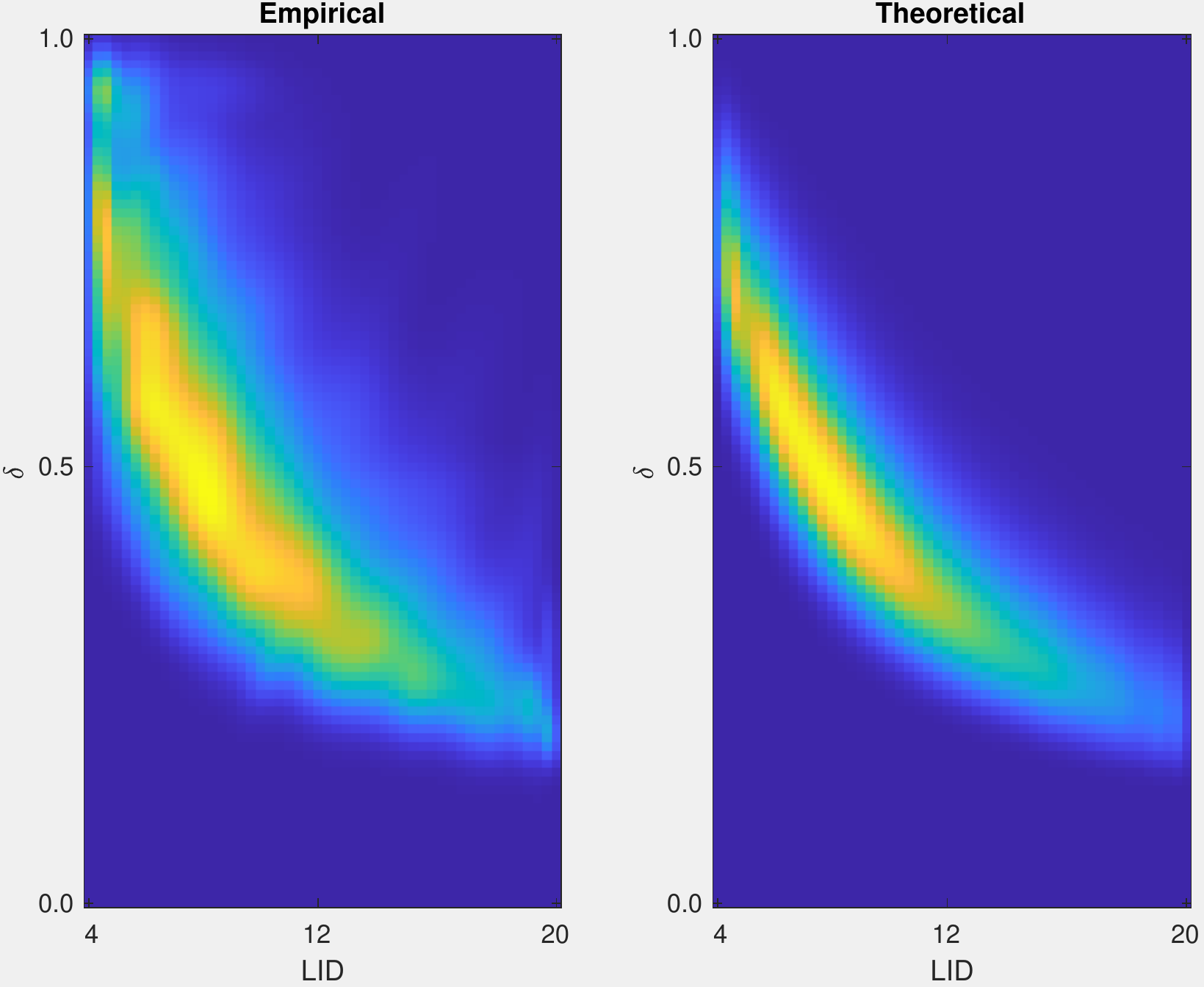}
\includegraphics[width = 0.3\textwidth,height=\rr cm]{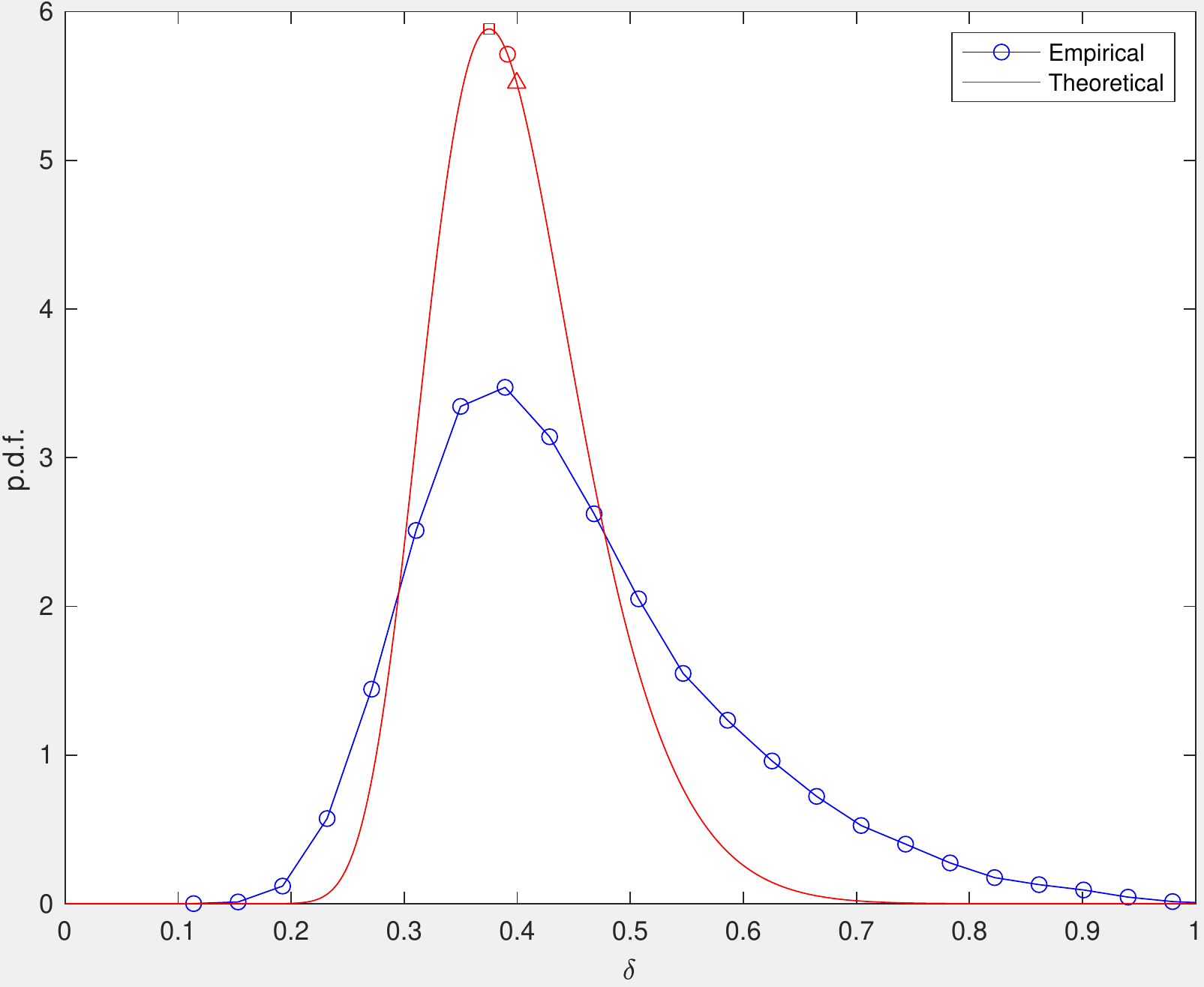}
\caption{\spacev\ with $\kx = 100$, $\kt=1$.}
\end{center}
\end{figure}

\begin{figure}[h]
\begin{center}
\includegraphics[width = 0.3\textwidth,height=\rr cm]{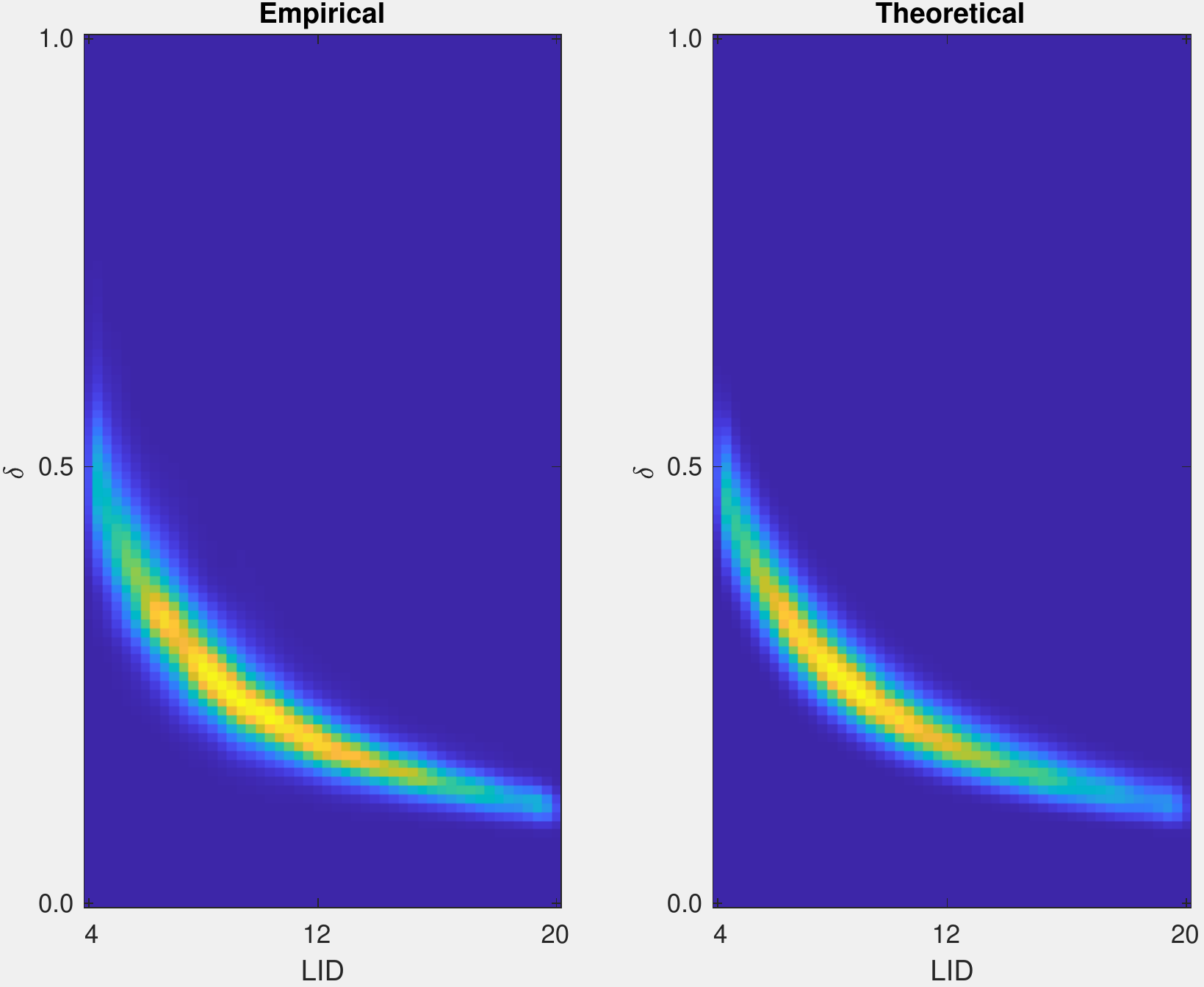}
\includegraphics[width = 0.3\textwidth,height=\rr cm]{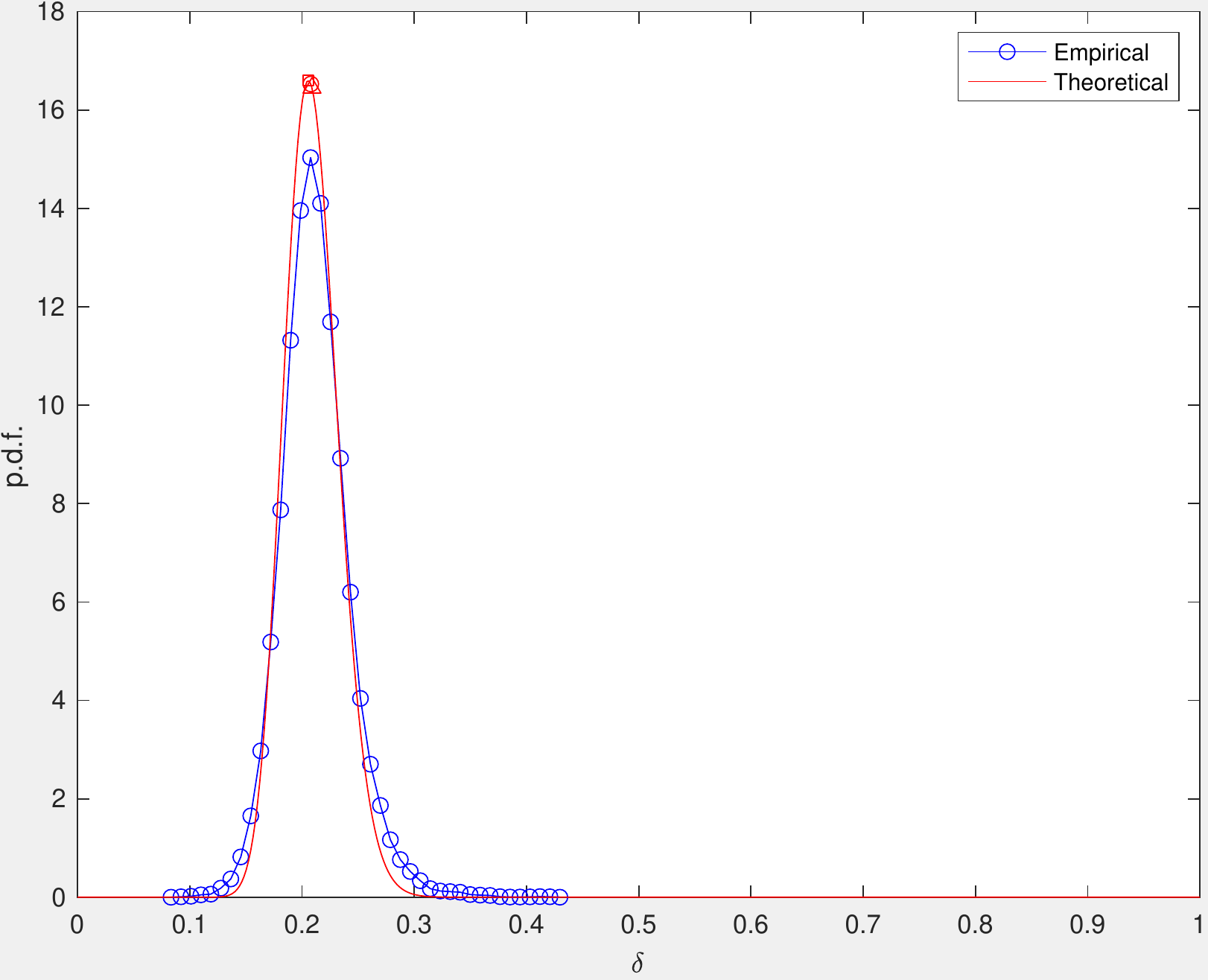}
\caption{\spacev\ with $\kx = 100$, $\kt=10$.}
\end{center}
\end{figure}

\begin{figure}[h]
\begin{center}
\includegraphics[width = 0.3\textwidth,height=\rr cm]{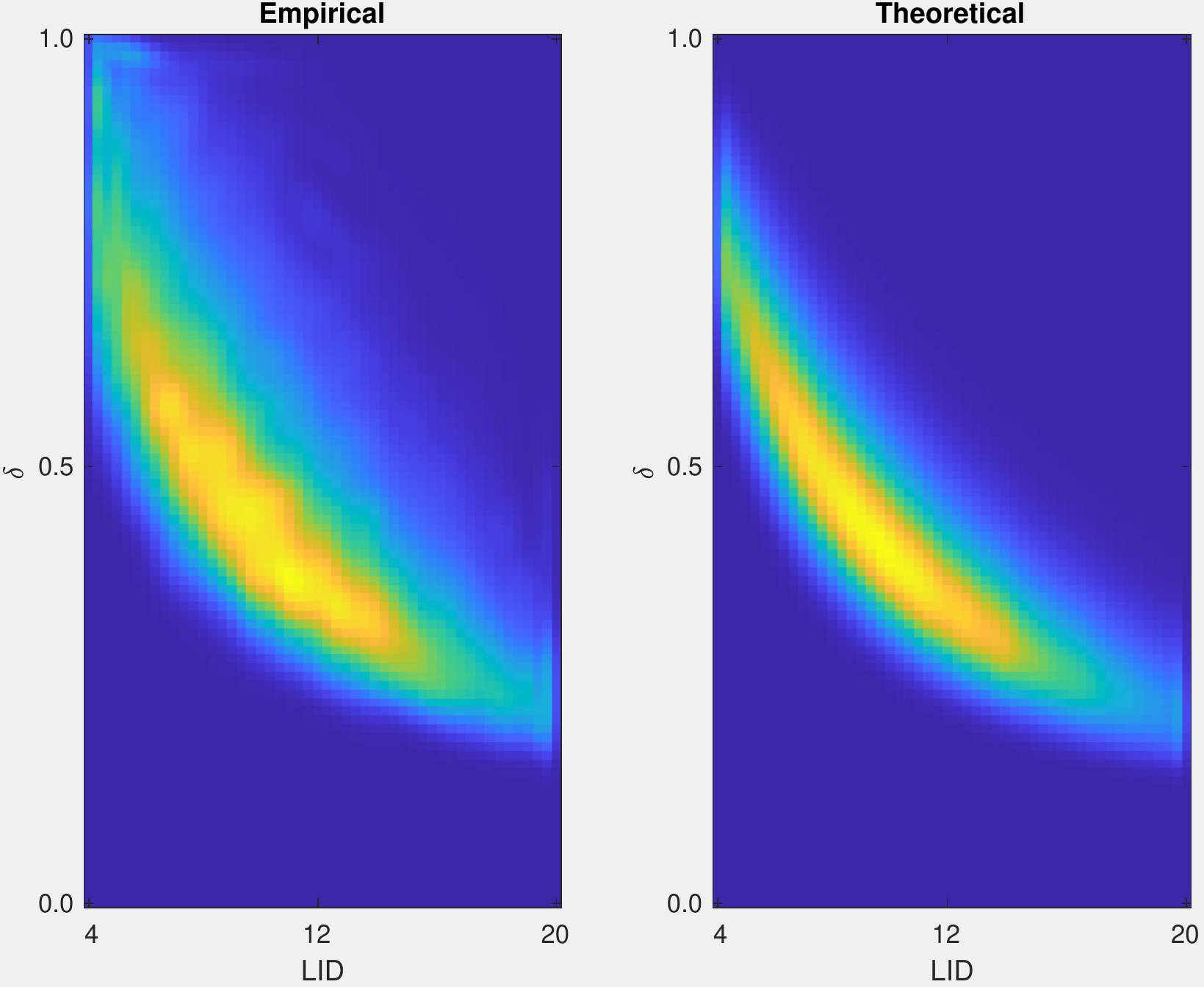}
\includegraphics[width = 0.3\textwidth,height=\rr cm]{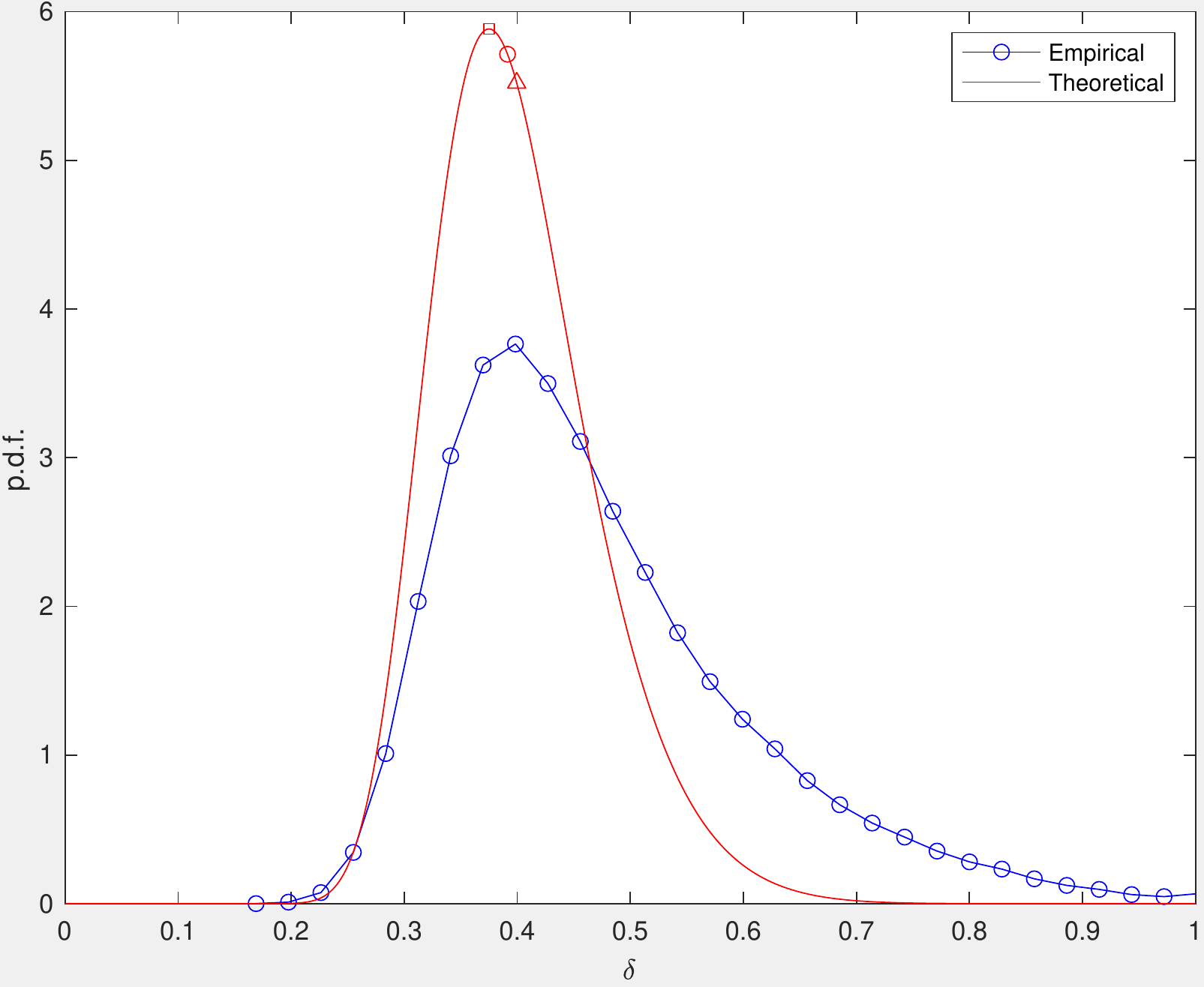}
\caption{\turing\ with $\kx = 100$, $\kt=1$.}
\end{center}
\end{figure}

\begin{figure}[h]
\begin{center}
\includegraphics[width = 0.3\textwidth,height=\rr cm]{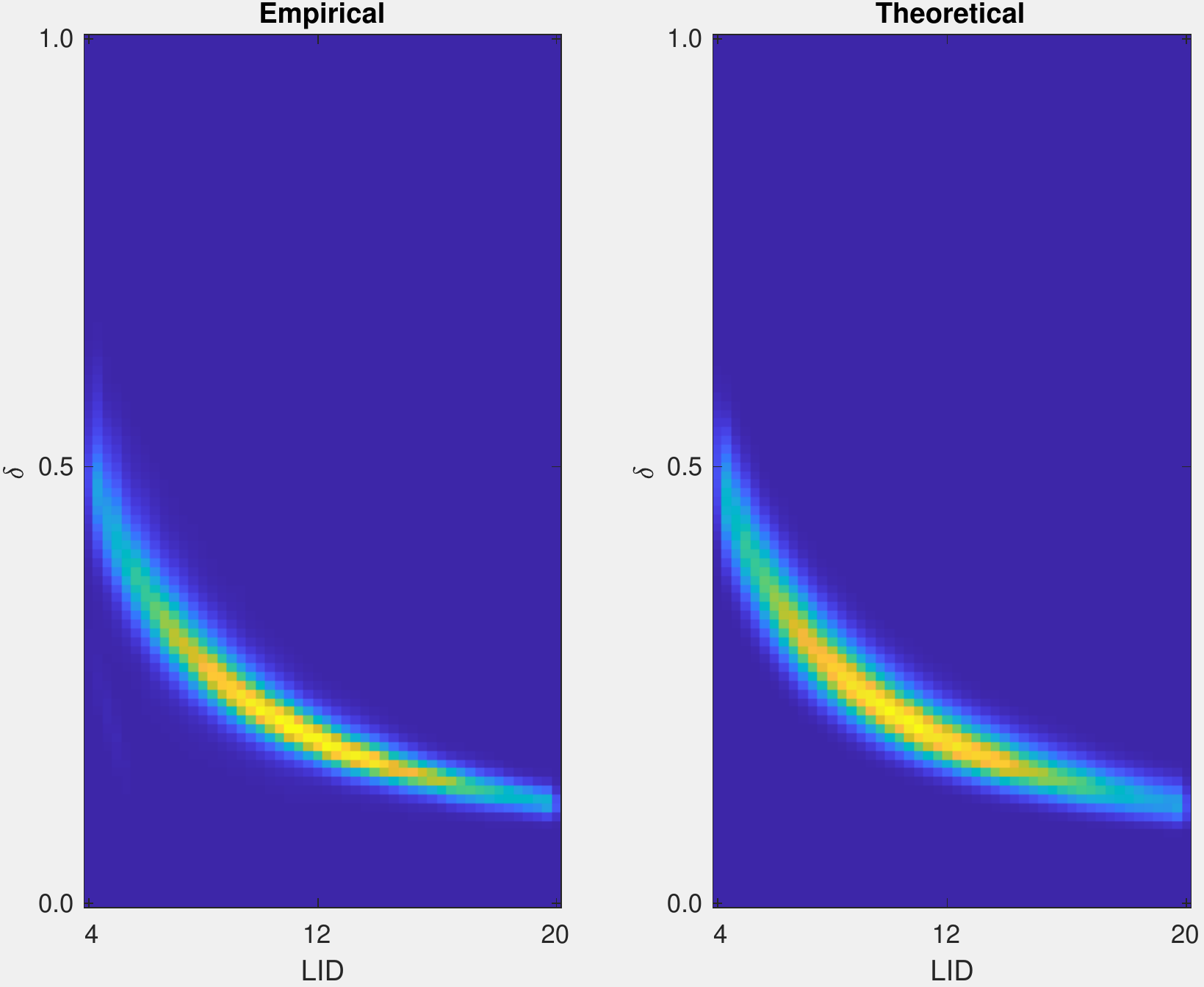}
\includegraphics[width = 0.3\textwidth,height=\rr cm]{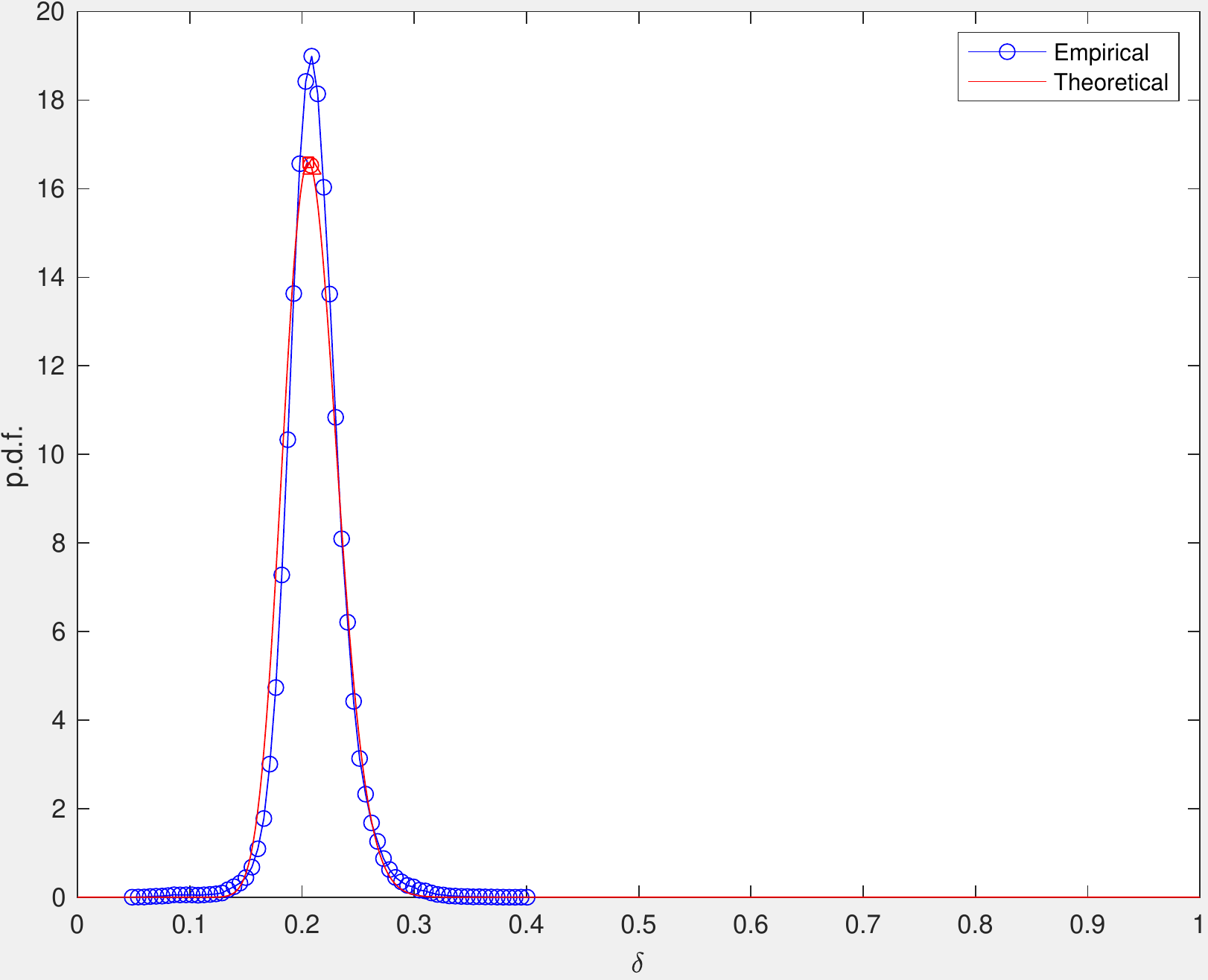}
\caption{\turing\ with $\kx = 100$, $\kt=10$.}
\end{center}
\end{figure}

\begin{figure}[h]
\begin{center}
\includegraphics[width = 0.3\textwidth,height=\rr cm]{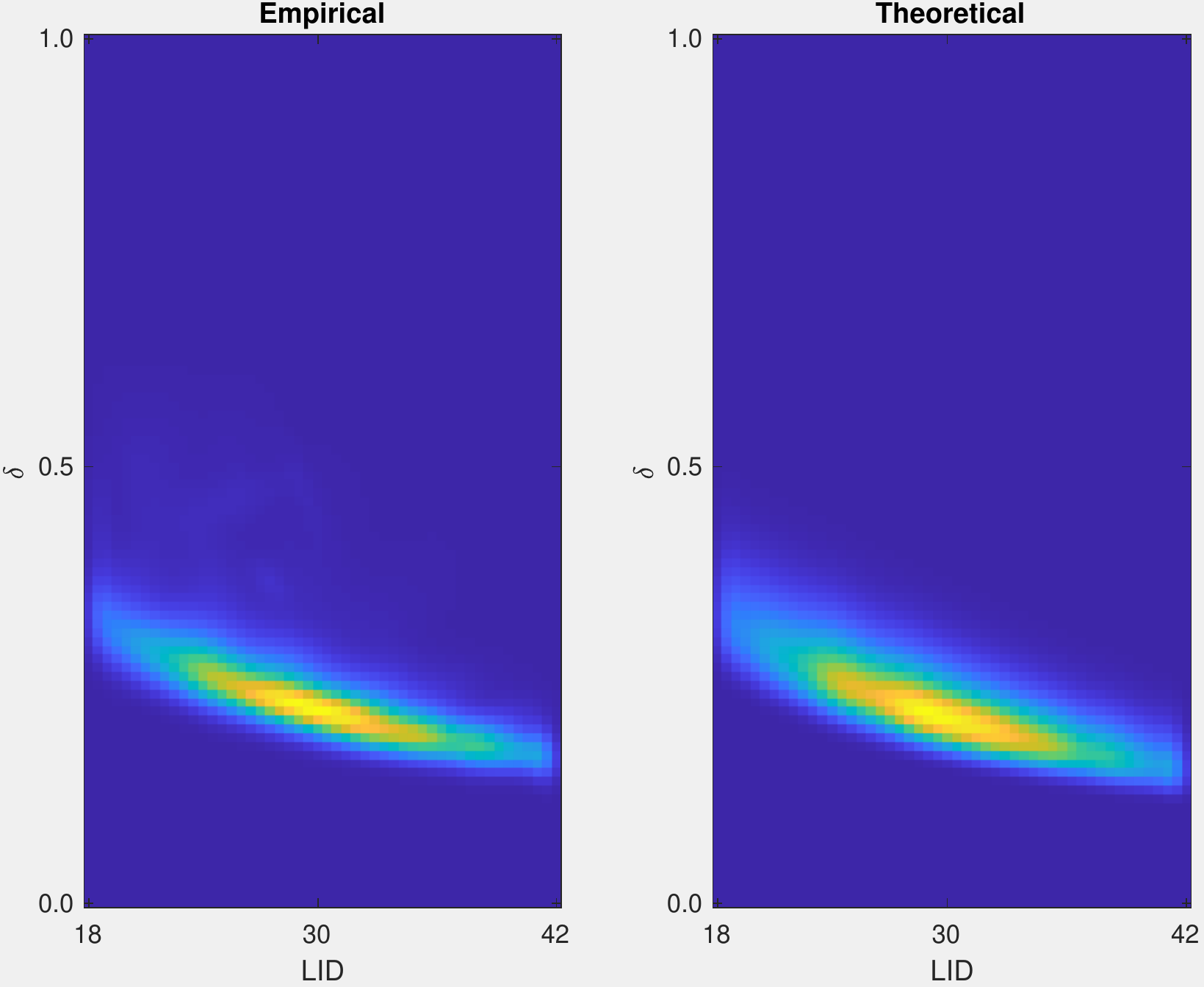}
\includegraphics[width = 0.3\textwidth,height=\rr cm]{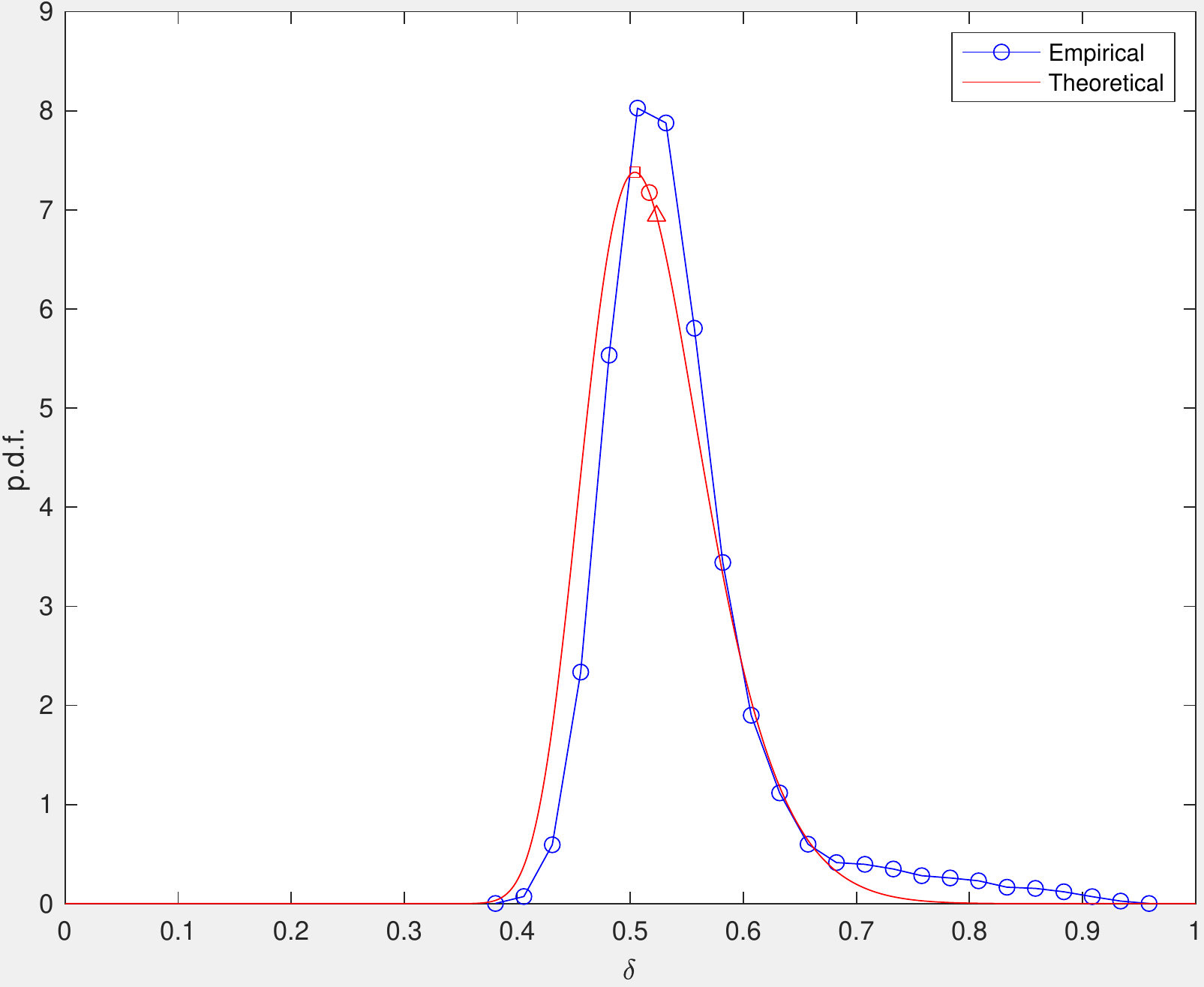}
\caption{\cifar\ with $\kx = 1000$, $\kt=1$.}
\end{center}
\end{figure}

\begin{figure}[h]
\begin{center}
\includegraphics[width = 0.3\textwidth,height=\rr cm]{./matlab/fig/Asymptotics/ImageNet/PDF_joint_1.pdf}
\includegraphics[width = 0.3\textwidth,height=\rr cm]{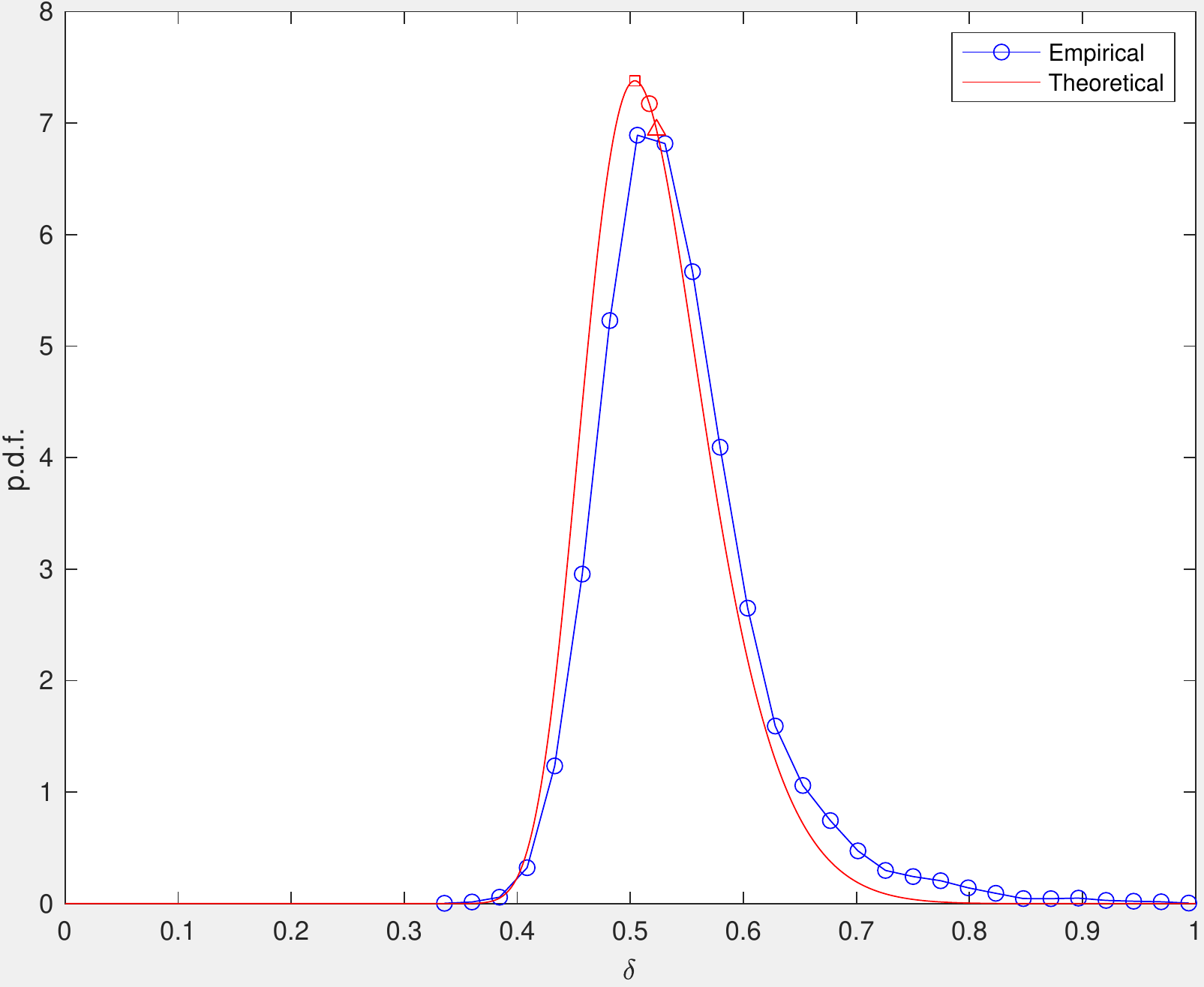}
\caption{\imagenet\ with $\kx = 1000$, $\kt=1$.}
\end{center}
\end{figure}

\end{document}